\documentclass{article}

    \PassOptionsToPackage{numbers, compress}{natbib}
 \usepackage[preprint]{neurips_2026}

\usepackage[utf8]{inputenc} 
\usepackage[T1]{fontenc}    
\usepackage{hyperref}       
\usepackage{url}            
\usepackage{booktabs}       
\usepackage{amsfonts}       
\usepackage{nicefrac}       
\usepackage{microtype}      
\usepackage{xcolor}         

\usepackage{hyperref}       
\usepackage{url}            
\usepackage{booktabs}       
\usepackage{amsfonts}       
\usepackage{nicefrac}       
\usepackage{microtype}      
\usepackage{lipsum}
\usepackage{fancyhdr}       
\usepackage{graphicx}       
\usepackage[numbers]{natbib}

\usepackage{xcolor}
\usepackage{tcolorbox}
\tcbuselibrary{breakable}
\usepackage{mathtools}
\usepackage{hyperref} 
\usepackage{cleveref}
\usepackage{algorithm}
\usepackage[noend]{algorithmic}
\usepackage{graphicx}
\usepackage{textcomp}
\usepackage{multirow}
\usepackage{xspace}
\usepackage{listings}
\usepackage{caption}
\usepackage{amsthm}
\usepackage{float}
\usepackage{soul}
\usepackage{upquote}
\usepackage{listings}
\usepackage{enumitem}
\usepackage{extarrows}
\usepackage{stmaryrd}
\usepackage{bbm}
\usepackage{subcaption}
\usepackage{makecell}
\usepackage{array}
\usepackage{dsfont}
\usepackage{colortbl}
\usepackage{tikz}
\usepackage{wrapfig}
\usepackage{amssymb}
\usepackage{diagbox}
\usepackage{bm}
\usepackage{fvextra}
\usepackage{thmtools}

\DeclareMathOperator*{\argmin}{argmin}

\theoremstyle{plain}

\theoremstyle{definition}

\theoremstyle{remark}

\newcommand{\J}{\mathcal{J}}

\newcommand{\cL}{\mathcal{L}}

\newcommand{\N}{\mathcal{N}}

\newcommand{\cP}{\mathcal{P}}
\newcommand{\Q}{\mathcal{Q}}

\newcommand{\T}{\mathcal{T}}

\newcommand{\tsf}[1]{\textsf{#1}}

\title{Capturing LLM Capabilities via Evidence-Calibrated Query Clustering}

\author{Fangzhou Wu$^{1}$, Sandeep Silwal$^{1,\dagger}$, Qiuyi (Richard) Zhang$^{2,\dagger}$ \\
$^{1}$University of Wisconsin--Madison, $^{2}$Elorian AI\\
\texttt{fwu89@wisc.edu, silwal@cs.wisc.edu, richard@elorian.ai} \\
}

\begin{document}

\maketitle

\begingroup
\renewcommand{\thefootnote}{\fnsymbol{footnote}}
\footnotetext[2]{The remaining authors are listed alphabetically.}
\endgroup

\begin{abstract}
    Query clustering organizes queries into groups that reflect shared latent capability demands, enabling capability-aware LLM evaluation.
    Existing clustering methods, which primarily rely on semantic taxonomies or embeddings, often fail to capture such latent capability requirements due to a misalignment between surface-level semantics and actual model performance. 
    We propose \tsf{ECC}, an algorithm that calibrates prior semantic embeddings using limited posterior model comparisons to bridge the gap between surface-level semantics and latent capability requirements. 
    \tsf{ECC} characterizes each cluster through a capability profile parameterized by a Bradley-Terry model and uses trainable mixture weights 
    to accommodate queries with mixed capability demands, jointly learning a {flexible}, capability-aware clustering structure that supports query-specific inference of LLM capabilities.
    Extensive quantitative and qualitative evaluations demonstrate that \tsf{ECC} significantly improves LLM capability ranking quality, outperforming human-labeled and embedding-based baselines by an average of \textbf{17.64} and \textbf{18.02} percentage points, respectively, and proves effective in downstream tasks such as query routing.
    Our code is available at \url{https://github.com/fzwark/ECC}.
\end{abstract}

\section{Introduction}

Large Language Models trained on extensive corpora exhibit broad domain knowledge and diverse capabilities, as evaluated by standard benchmarks~\citep{agarwal2025gpt, team2024gemini}. 
Since different LLMs excel in different tasks (i.e. \emph{query-dependent}), accurately assessing their strengths is essential for reliable deployment in domain-specific and personalized settings~\citep{frieder2023mathematical, novikov2025alphaevolve}.
However, coarsely curated datasets, or groups of queries, often yield noisy evidence that is insufficient for characterizing latent LLM capabilities, especially when queries within a group differ widely in their capability demands.
This motivates clustering queries into subsets that share similar underlying capability demands, each corresponding to a distinct capability profile, which can enable a capability-aware LLM evaluation~\citep{wang2024mmlu,hendrycks2020measuring}.

Standard query ``clustering'' methods typically rely either on static and heuristic human-labeled semantic taxonomies, such as ``Mathematics'' or ``Chemistry''~\citep{zhang2024mathverse, guo2023can}, or on unsupervised grouping based on semantic embeddings~\citep{arenahard2024, miroyan2025search}.
While intuitive, \textbf{such surface-level semantic clusters often fail to capture the true latent capability requirements of queries} (Figure~\ref{fig:misalign}):
(i) A single label (e.g., ``Mathematics'') may encompass queries requiring vastly different levels of underlying capabilities, from rote recall to multi-step deductive derivation, or even a combination of distinct capability requirements. 
(ii) Conversely, queries requiring the same underlying capabilities, such as multi-step deductive derivation, may be scattered across different semantic clusters due to superficial topical differences. 
This misalignment with the true capability distributions can limit, or even degrade, the generalizability of model capability estimates to unseen queries (Figure~\ref{fig:human}), highlighting the insufficiency of prior semantic signals.
Although recent works leverage ``LLM-as-a-judge''~\citep{zheng2023judging} to extract explicit skill labels from model responses and build capability structures~\citep{moayeri2025unearthing, tian2025skillverse}, the resulting partitions of queries are typically hard and discrete, and thus do not naturally accommodate queries that require a complex mixture of capabilities.
Furthermore, these approaches rely on exhaustively evaluating each query across numerous models to determine cluster assignments, which incurs substantial computational cost with LLM-based judges.

These fundamental limitations motivate us to propose \emph{\textbf{E}vidence-\textbf{C}alibrated \textbf{C}lustering} (\tsf{ECC}), a novel algorithm that efficiently learns underlying capability-aware cluster structure by directly incorporating evidence of model performance, which we term \emph{posterior evidence}, into the clustering process.
Rather than relying solely on \emph{prior} semantic embeddings, \tsf{ECC} incorporates limited pairwise model comparisons over query responses as \emph{sample-efficient} signals of LLM capabilities, characterizing each cluster with a capability profile parameterized by a Bradley-Terry model~\citep{bradley1952rank} optimized over these comparisons. 
To accommodate queries requiring mixed capabilities, \tsf{ECC} introduces trainable \emph{soft responsibilities} to weight the contribution of query embeddings and associated model comparisons across clusters, quantifying how strongly each query aligns with the capabilities represented by each cluster.
To reconcile prior semantics with posterior evidence, \tsf{ECC} jointly optimizes cluster centroids, capability profiles, and soft responsibilities through alternating updates.

For unseen queries, \tsf{ECC} employs \emph{probe-informed mixture inference}, which leverages a single probe signal (i.e., a single pairwise model comparison) alongside the query embedding to derive a query-specific, capability-aware ranking over LLMs (see~\Cref{sec:method}).

\begin{figure}[t]
    \centering
    \includegraphics[width=0.6\linewidth]{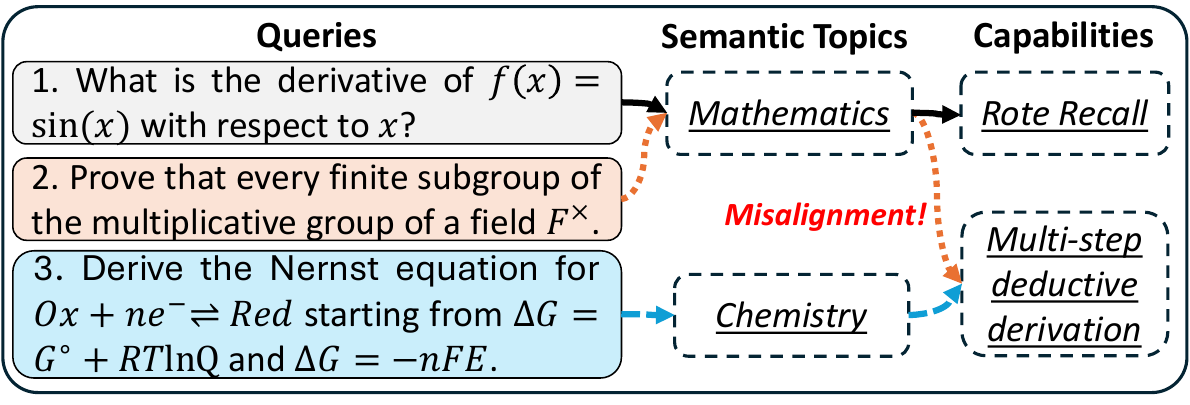}
    \caption{
    Misalignment between surface-level semantics and latent capability demands. 
    }
    \vspace{-1em}
    \label{fig:misalign}
\end{figure}

We validate the effectiveness of \tsf{ECC} through extensive quantitative experiments across diverse benchmarks, demonstrating average improvements in ranking quality on unseen queries of \textbf{17.64} percentage points (pp) over human-labeled clustering and \textbf{18.02} pp over embedding-only clustering. 
We analyze the clusterings qualitatively to confirm that incorporating our posterior evidence yields interpretable capability-aware clusters, separating semantically similar queries with divergent capability needs and merging those with aligned needs.
Finally, we apply \tsf{ECC} to two representative downstream tasks: 
(1) {guided optimal query routing}, where it improves the response quality of routed queries by \textbf{16.6\%} on average over embedding-only clustering, and 
(2) {sample-efficient new model ranking}, where it achieves \textbf{16.3} pp higher ranking quality under a limited comparison budget.

In summary, our contributions are:

\begin{enumerate}[leftmargin=*, noitemsep]
    \item We propose \tsf{ECC}, a novel query clustering algorithm that integrates posterior model comparisons with prior embeddings to overcome the misalignment between surface-level semantics and latent capability requirements.
    \item By characterizing each cluster with a capability profile parameterized by a Bradley-Terry model, \tsf{ECC} enables sample-efficient estimation of relative model strengths from limited pairwise comparisons.
    \item \tsf{ECC} captures the mixed capability requirements of queries by introducing trainable responsibilities, yielding a flexible, capability-aware clustering structure that enables query-specific inference of LLM capabilities.
    \item Extensive quantitative and qualitative evaluations show that \tsf{ECC} produces more capability-coherent clusters and significantly improves capability-aware ranking quality, outperforming baselines in two downstream tasks.
\end{enumerate}

\subsection{Related Work}
\noindent\textbf{Capability-Aware Clustering for LLM Evaluation.}
Recent works move beyond human-defined taxonomies~\citep{wang2019superglue, ying2024mmt, liu2024mmbench, srivastava2023beyond, al2024unibench} by inducing latent capability structure from evaluation data for finer-grained LLM evaluation~\citep{tian2025skillverse, moayeri2025unearthing, wei2025kcluster, luo2024llm, murahari2024qualeval}.
For example, SkillVerse~\citep{tian2025skillverse} builds a hierarchical skill dendrogram using LLM-as-a-judge critiques, while other work extracts task attributes or knowledge components to guide clustering and improve performance prediction~\citep{moayeri2025unearthing, wei2025kcluster}.
In contrast to these methods that derive fixed groupings from explicit attribute representations, \tsf{ECC} learns soft capability-aligned clusters that support query-conditional inference over LLM capabilities.

\noindent\textbf{Sample-Efficient LLM Evaluation.}
Full-benchmark evaluation is costly, as it requires scoring all models over thousands of items~\citep{li2025adaptivetestingllmevaluation, chiang2024chatbot}.
Recent work reduces this cost via comparison-based ranking and informative item selection~\citep{chiang2024chatbot, son2025arena, li2025adaptivetestingllmevaluation, kipnis2025metabench, baker2001basics, song2025irt}.
For example, Chatbot Arena~\citep{chiang2024chatbot} fits Bradley-Terry models~\citep{bradley1952rank} to pairwise LLM comparisons to produce scalable global leaderboards, while ATLAS~\citep{li2025adaptivetestingllmevaluation} applies IRT-based adaptive testing to estimate global model ability with far fewer benchmark questions.
These methods primarily target a \emph{global} leaderboard shared across queries, whereas \tsf{ECC} 
infers a {query-specific} mixture ranking over LLMs, enabling fine-grained adaptation across heterogeneous queries.
See extended related work in~\Cref{app:related}.

\subsection{Bradley-Terry Model for LLM Ranking.}
The Bradley-Terry (BT) model is a widely used framework for estimating latent strength scores of items from pairwise comparison data~\citep{bradley1952rank}.
In the context of LLM evaluation, it has recently been adopted to rank a set of language models from limited 
pairwise comparisons over their responses to given queries~\citep{chiang2024chatbot}.
Formally, let $\Q$ denote a set of queries and let there be $M$ different LLMs. For each query $q \in \Q$, we assume access to a set of pairwise model comparisons $\N_q=\{(i,j,y)\}$, where $i,j \in [M]$ refer to two LLMs $m_i$ and $m_j$, and $y\in\{0, 1\}$ denotes the comparison outcome over their responses to query $q$, with $ y = 1$ indicating that $m_j$ is preferred over $m_i$, and $y=0$ otherwise. 
In the BT model, the probability that $m_j$ is preferred over $m_i$ is defined using the logistic function $\sigma(\cdot)$: 
\vskip -0.17in
\begin{equation}
    P_{\bm{\theta}}(m_j \succ m_i) = \sigma(\theta_j - \theta_i) = \frac{1}{1 + e^{-(\theta_j - \theta_i)}}, 
\end{equation}
\vskip -0.08in
where $\bm{\theta} \in \mathbb{R}^M$ denotes the vector of latent strengths for the $M$ models.
To learn $\bm{\theta}$, we minimize the BT model loss (neg. log-likelihood) over all pairwise comparisons across all queries:
$\ell(\bm{\theta}) := -\sum_{q\in\Q} \sum_{(i, j, y) \in \N_q} \log P (y\mid i, j; \bm{\theta})$,
where the log-probability $ \log P (y\mid i, j; \bm{\theta})$ is defined as $ y \log \sigma(\theta_j - \theta_i) + (1-y) \log \sigma(\theta_i - \theta_j).$

\noindent\textbf{Capability-Aware Ranking Quality.}
We quantify the quality of a learned model capability ranking $\bm{\theta}$ for a query $q$ by how well it explains the observed pairwise comparison outcomes, operationalized via the per-query BT loss:
\vskip -0.2in
\begin{equation}
    \label{eq:btloss}
    \ell_{\text{comp}}(q; \bm{\theta}) = -\frac{1}{|\N_q|} \sum_{(i, j, y) \in\N_q} \log P(y\mid i,j;\bm{\theta}).
\end{equation}
\vskip -0.1in
Lower values of $\ell_{\text{comp}}(q; \bm{\theta})$ indicate a better fit to the observed evidence and thus higher ranking quality.
Crucially, this metric is probability-aware: it evaluates not only the ordinal consistency of the ranking with observed comparisons but also the predictive confidence \citep{frick2025prompttoleaderboard}.

\section{Methodology}\label{sec:method}

The main intuition behind \emph{\textbf{E}vidence-\textbf{C}alibrated \textbf{C}lustering} (\tsf{ECC}) stems from a key limitation in existing query clustering methods that only rely on prior semantic embeddings: posterior evidence with informative model capability signals, such as pairwise model comparisons, is only used after clustering. 
This creates a fundamental gap: capability evidence cannot influence cluster formation, resulting in semantically coherent clusters that may still be misaligned with latent capability demands.
\tsf{ECC} instead \emph{evaluates first, then clusters}, where it first collects comparison evidence on unclustered queries and then uses it to calibrate and reshape the embedding-based semantic structure.
Concretely, given an input number of clusters $K$, for each query $q\in\Q$, \tsf{ECC} jointly leverages its semantic embedding $\bm{x}_q$ and its associated pairwise comparisons $\N_q$ to learn a flexible latent capability clustering that enables reliable query-specific capability inference of LLMs.

\noindent\textbf{System Parameters $\bm{\Theta}$ and Responsibilities $\cP$.}
To achieve this, \tsf{ECC} characterizes each specific cluster $k \in [K]$ using two complementary components:
(i) a \textbf{\emph{semantic centroid}} $\bm{c}_k$, capturing the {prior} semantic signals in the embedding space, 
and (ii) a \textbf{\emph{capability profile}} $\bm{\theta}^{(k)} \in \mathbb{R}^M$ parameterizing a BT model that represents the posterior capability strengths of the $M$ LLMs on queries associated with the cluster.
Together, these define the overall \textbf{\emph{system parameters}}: $\bm{\Theta} = \{\bm{\theta}^{(k)}, \bm{c}_k\}_{k=1}^K$.

Since a single query may involve multiple capability demands, \tsf{ECC} adopts soft assignments rather than hard clustering by introducing a set of \emph{\textbf{responsibilities}} $\cP = \{r_{qk}\}$.
For each query $q$, the responsibilities satisfy $r_{qk}\geq 0$ and $\sum_{k=1}^K r_{qk}=1$, where $r_{qk}$ measures how strongly $q$ aligns with the capability profile represented by cluster $k$.
Together, these cluster-level profiles define a \emph{fine-grained capability structure}: each cluster serves as a basic capability unit, and each query forms its own capability representation by softly combining these units through its responsibilities.

\noindent\textbf{Alternating Optimization.}
Unlike purely embedding-based clustering methods (e.g., K-Means~\citep{mcqueen1967some}) that update only semantic centroids, \tsf{ECC} follows a two-stage alternating optimization procedure, where it iteratively updates (i) the system parameters $\bm \Theta$ and (ii) the soft responsibilities $\cP$, allowing both semantic priors and posterior evidence to jointly shape the resulting clusters.

\emph{\underline{Optimization Objective}}.
For each query $q$ and cluster $k$, we define a \emph{combined loss}: 
\vskip -0.17in
\begin{equation}
\label{eq:total_loss}
    \ell(q,k;\bm \Theta) := \ell_{\text{comp}}(q;\bm{\theta}^{(k)}) + \lambda \ell_{\text{emb}}(q;\bm{c}_k),
\end{equation}
\vskip -0.13in
where $\ell_{\text{comp}}(q;\bm{\theta}^{(k)})$ measures how well the capability profile of cluster $k$ fits the pairwise comparisons of $q$ (see~\Cref{eq:btloss}), and $\ell_{\text{emb}}(q;\bm{c}_k):= -\bm{x}_q^\top \bm{c}_k$ measures the semantic mismatch between the query embedding and the cluster centroid. 
The trade-off parameter $\lambda$ balances the influence of the posterior evidence and the prior embedding signal. Instead of directly minimizing~\Cref{eq:total_loss} via hard assignments that allocate each query to a single cluster, \tsf{ECC} adopts a soft optimization approach.
It jointly optimizes the system parameters $\bm \Theta$ and soft responsibilities $\cP$ by minimizing the following \emph{entropy-regularized expected loss}:
\vskip -0.2in
\begin{equation}
    \cL({\cP}, \bm{\Theta}) := \sum_{q \in \Q} \sum_{k=1}^K  r_{qk} {\ell(q,k;\bm \Theta)} - \T \sum_{q \in \Q} H(\bm{r}_q),
    \label{eq:objective}
\end{equation}
\vskip -0.15in
where $H(\bm{r}_q) = - \sum_{k=1}^K r_{qk} \log r_{qk}$ is the Shannon entropy of the assignment distribution, and temperature $\T>0$ controls the strength of the entropy regularization.

\begin{wrapfigure}[20]{R}{0.53\textwidth}
\vspace{-2.3em}
\begin{minipage}{\linewidth}
\begin{algorithm}[H]
\caption{\small \tsf{ECC}: Evidence-Calibrated Clustering}
\label{alg:ecc}
\small
\begin{algorithmic}[1]
    \STATE {\textbf{Input}:} Query set $\Q$ with embeddings $\{\bm{x}_q\}$,
    model comparisons $\N = \{\N_q\}$, number of clusters $K$.
    \STATE Initialize responsibilities $\cP = \{r_{qk}\}$ randomly.
    \STATE Normalize embeddings: $\bm{x}_q \leftarrow \bm{x}_q / \| \bm{x}_q\|_2$, $\forall q\in \Q$.

    \REPEAT
        \STATE {\ttfamily\textcolor{red}{\(\triangleright\) {/* Optimize System Parameters */}}}
        \FOR{$k \in [K]$}
            \STATE {\ttfamily\textcolor{red}{\(\triangleright\) {/* Update Centroids $\bm{c}_k$ */}}}
            \STATE $\bm{c}_k \leftarrow
            \mathrm{norm}\!\left(
            \left(\sum_{q \in \Q} r_{qk} \bm{x}_q\right)/
            \sum_{q \in \Q} r_{qk}
            \right)$

            \STATE {\ttfamily\textcolor{red}{\(\triangleright\) {/* Update BT Model $\bm{\theta}^{(k)}$ */}}}
            \STATE $\bm{\theta}^{(k)} \leftarrow
            \operatorname*{\argmin}_{\bm{\theta}}
            \sum_{q \in \Q} r_{qk} \ell_{\text{comp}}(q; \bm{\theta})$
        \ENDFOR

        \STATE {\ttfamily\textcolor{red}{\(\triangleright\) {/* Optimize Responsibilities */}}}
        \FOR{$q \in \Q$}
            \FOR{$k \in [K]$}
                \STATE $h_{qk} \leftarrow
                -\left[
                \ell_{\text{comp}}(q;\bm{\theta}^{(k)})
                + \lambda \ell_{\text{emb}}(q;\bm{c}_k)
                \right]/\T$
            \ENDFOR

            \STATE $r_{qk} \leftarrow
            \frac{\exp(h_{qk})}
            {\sum_{j=1}^K \exp(h_{qj})},
            \quad \forall k\in[K]$
        \ENDFOR

    \UNTIL{Convergence}

    \STATE {\bfseries Output:} Responsibilities $\cP$ and system parameters $\bm{\Theta}$.
\end{algorithmic}
\end{algorithm}
\end{minipage}
\end{wrapfigure}

\emph{\underline{Stage I: Updating $\bm{\Theta}$ given $\cP$}}. 
Fixing the current responsibilities $\cP$, the entropy term in~Objective~\ref{eq:objective} becomes constant, and the optimization reduces to
$\min_{\bm{\Theta}} \sum_{q,k} r_{qk} \ell(q,k;\bm \Theta)$.
This objective separates across clusters and yields two closed-form sub-updates.
For each cluster $k$, \tsf{ECC} first updates its semantic centroid $\bm{c}_k$ by minimizing $\sum_q r_{qk}\,\ell_{\text{emb}}(q;\bm c_k)$, which results in a normalized posterior-weighted mean (line 8 in~\Cref{alg:ecc}):  $  \bm{c}_k \leftarrow \frac{\sum_{q \in \Q} {r_{qk}} \bm{x}_q}{\sum_{q \in \Q} {r_{qk}}}, \quad \bm{c}_k \leftarrow \bm{c}_k / \|\bm{c}_k\|_2.$

Next, it updates the capability profile of each cluster $\bm \theta^{(k)}$ by minimizing the weighted BT loss:  
$  \bm{\theta}^{(k)} = \operatorname*{\argmin}_{\bm{\theta}} \sum_{q \in \Q} {r_{qk}} \ell_{\text{comp}}(q; \bm{\theta}).$
The responsibilities enable each query to inform multiple clusters to varying degrees, which is essential for modeling queries that require multiple capabilities.

\emph{\underline{Stage II: Updating $\cP$ given $\bm{\Theta}$}}. 
In the second stage, \tsf{ECC} updates $\cP$ while keeping the current system parameters $\bm{\Theta}$ fixed.
Since~Objective~\ref{eq:objective} is separable across queries, each $\bm r_q$ admits a closed-form update (lines 12-15 in~\Cref{alg:ecc}):
\vskip -0.15in
\begin{equation}
\label{eq:posterior_update}
    r_{qk} \leftarrow \frac{\exp\left(-\ell(q,k;\bm \Theta)/\T\right)}{\sum_{j=1}^K \exp\left(-\ell(q,j;\bm \Theta)/\T\right)}.
\end{equation}
\vskip -0.09in

\emph{\underline{Convergence}}.
\tsf{ECC} terminates at iteration $t$ when the responsibilities and centroids stabilize: the average KL divergence between successive responsibilities $\frac{1}{|\Q|}\sum_{q}\mathrm{KL}(\bm r_q^{(t)}\|\bm r_q^{(t-1)})$ and the centroid shift $\max_k\|\bm c_k^{(t)}-\bm c_k^{(t-1)}\|_2$ fall below predefined thresholds.

\noindent\textbf{Probe-Informed Mixture Inference.}
Since \tsf{ECC} learns weighted mixture capability clusters, inferring LLM capabilities for a new query $q$ first requires estimating its soft assignment $\bm{r}_q$ across clusters.
A naive approach is to infer $\bm{r}_q$ using embeddings alone; however, semantic proximity does not reliably imply capability proximity.
Furthermore, this embedding-only inference ignores the fact that the learned cluster structure is jointly shaped by both embeddings and comparison evidence during clustering, resulting in a signal mismatch between clustering and inference.
To reconcile this mismatch in a lightweight manner, we introduce a single \emph{probe} signal during inference:
for each query, we randomly collect \textbf{only one pairwise model comparison} and use it to calibrate $\bm{r}_q$ via the same entropy-regularized update as in \Cref{eq:total_loss,eq:posterior_update}, with $\ell_{\text{comp}}$ computed on the probe only. 
Given the calibrated $\bm{r}_q$, \tsf{ECC} estimates pairwise preferences using the following weighted mixture:
\vskip -0.18in
\begin{equation}
\label{eq:probe_mixture_pred}
P(m_j \succ m_i \mid q) = \sum_{k=1}^K r_{qk} \sigma \left(\bm{\theta}^{(k)}_j-\bm{\theta}^{(k)}_i\right).
\end{equation}
\vskip -0.1in
Although it only costs a single comparison, the probe provides high-leverage capability evidence, with its effectiveness empirically demonstrated in~\Cref{fig:main}. We prove the following convergence guarantee in Appendix~\Cref{app:proof}:
\begin{restatable}[Monotone convergence of \tsf{ECC}]{theorem}{MainTheorem}
\label{thm:convergence}
Assume that the objective $\cL(\cP,\bm{\Theta})$ is finite at initialization.
Then, the objective values $\{\cL^{(t)}\}_{t \ge 0}$ produced by \tsf{ECC} are monotonically non-increasing:
    $\cL(\cP^{(t+1)}, \bm \Theta^{(t+1)}) \leq \cL(\cP^{(t)}, \bm \Theta^{(t)})$.
Moreover, $\{\cL^{(t)}\}$ converges to a finite limit. 
\end{restatable}

\section{Evaluation}\label{sec:eval}

\begin{figure}[t]
    \centering
    \includegraphics[width=0.6\linewidth]{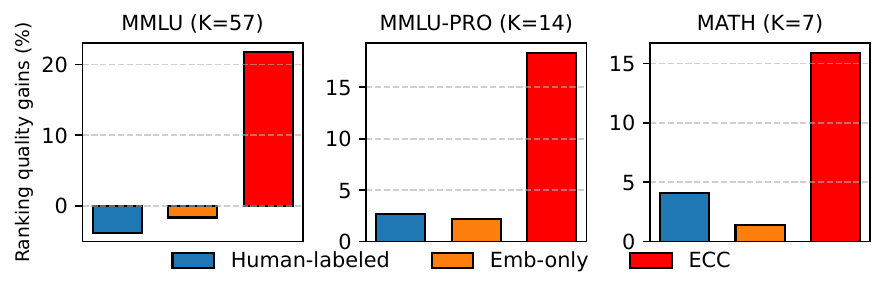}
    \caption{
    Ranking quality gains of different clustering methods on three datasets with human taxonomies.
    \tsf{ECC} consistently outperforms the human-labeled clustering.  
    }
    \label{fig:human}
    \vskip -0.1in
\end{figure}

\begin{figure*}[t]
    \centering
    \includegraphics[width=0.85\linewidth]{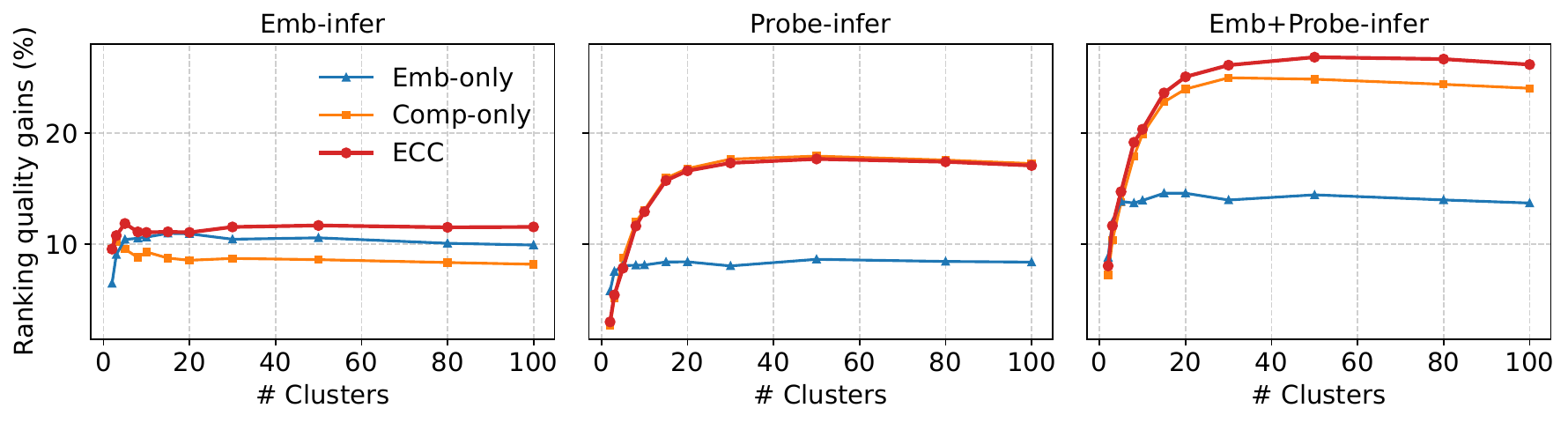}
    \caption{
   Ranking quality gains of different clustering methods on SPROUT across three inference signals and varying numbers of clusters $K$.
   \tsf{ECC} consistently achieves the best performance.
    }
    \label{fig:main}
    \vskip -0.2in
\end{figure*}

\begin{table*}[t]
\centering
\small
\caption{
Ranking quality gains (\%) for different \textbf{clustering methods (columns)} and \textbf{inference signals (rows)} across three benchmarks with $K=30$.
\tsf{ECC} consistently achieves superior performance across all inference signals.
}
\setlength{\tabcolsep}{3.5pt}
\label{tab:main}
\scalebox{0.85}
{
\begin{tabular}{lccccccccc}
\toprule
\multirow{2.5}{*}{\textbf{\makecell{Inference\\Signal}}} 
& \multicolumn{3}{c}{\textbf{SPROUT}} 
& \multicolumn{3}{c}{\textbf{RouterBench}} 
& \multicolumn{3}{c}{\textbf{Open LLM Leaderboard v2}} \\
\cmidrule(lr){2-4}\cmidrule(lr){5-7}\cmidrule(lr){8-10}
& \textbf{Emb-only} & \textbf{Comp-only} & \textbf{\tsf{ECC}}
& \textbf{Emb-only} & \textbf{Comp-only} & \textbf{\tsf{ECC}}
& \textbf{Emb-only} & \textbf{Comp-only} & \textbf{\tsf{ECC}} \\
\midrule
\textbf{Emb-infer} & 10.45 & 8.72 & 11.55 & 6.84 & 1.93 & 8.25 & 0.66 & -0.05 & 1.25 \\
\textbf{Probe-infer} & 8.04 & 17.67 & 17.32 & 5.19 & 24.74 & 24.19 & 1.88 & 20.15 & 19.81 \\
\textbf{Emb+Probe-infer} & 13.98 & 25.01 & \textbf{26.15} & 8.39 & 28.31 & \textbf{31.05} & 1.82 & 20.96 & \textbf{21.04}  \\
\bottomrule
\end{tabular}
}
\vskip -0.12in
\end{table*}

\noindent\textbf{Benchmarks.} 
We evaluate on three widely used multi-LLM benchmarks:
i) SPROUT~\citep{somerstep2025carrot},
ii) RouterBench~\citep{hu2024routerbench},
and iii) Open LLM Leaderboard v2~\citep{fourrier2024open}.
The queries in these benchmarks span diverse datasets, including MMLU~\citep{hendrycks2020measuring}, MMLU-Pro~\citep{wang2024mmlu}, and MATH~\citep{hendrycks2021measuring}.
For each benchmark, all model responses are scored by the same judge to ensure consistency.
We split queries into clustering/test sets with varying clustering proportions from 20\% to 90\%, using 80\% by default.
For each query, we decompose the full set of response scores into pairwise comparisons between models, using non-tied comparisons by default and considering tie-aware variants in the ablation study.
During clustering, we sample a varying number of model comparisons per query (from 2 to 24), with 7 used by default, and vary the number of clusters from 2 to 100, using 30 as the default.
For each test query, unless otherwise specified, we hold out one comparison as a probe for responsibility inference and evaluate on the remaining comparisons.
We use \texttt{bge-base-en-v1.5}~\citep{bge_embedding} for query embeddings, and fix $\lambda=3$ and $\T=0.1$ unless otherwise specified.

\noindent\textbf{Baselines.}
We consider the following clustering baselines:
(i) \emph{Human-labeled clustering}, which applies hard cluster partitions using the standard human-annotated taxonomies already provided by the benchmarks;
(ii) \emph{Embedding-only clustering (Emb-only)}, which uses only semantic embeddings for clustering, following the same soft update rule as \tsf{ECC}.
Both baselines represent prior semantic clustering structures.
We additionally consider a variant of \tsf{ECC}, \emph{\textbf{Comparison-only (Comp-only)}}, which uses only comparison signals by setting $\lambda=0$.
For Emb-only, Comp-only, and \tsf{ECC}, we use the same mixture ranking inference in~\Cref{sec:method} for a fair comparison.
We also include Prompt-to-Leaderboard (P2L)~\citep{frick2025prompttoleaderboard} as a non-clustering baseline.

\noindent\textbf{Inference Signals.}
Given the learned clusters, we consider three inference signals for estimating the LLM ranking for a new query:
(i) \emph{Emb-infer}, which only uses the embedding signal of the query;
(ii) \emph{Probe-infer}, which only uses a probe model comparison of the query; 
(iii) \emph{Emb+Probe-infer}, which combines both signals.
Emb+Probe-infer is used as the default unless stated otherwise.

\noindent\textbf{Metrics.}
We use two main metrics to evaluate different clustering methods:
\textbf{(1) Clustering performance.}
We evaluate the effectiveness of a capability clustering structure via its impact on \emph{capability-aware ranking quality} over unseen queries.
Specifically, we report the \emph{ranking quality gains}, defined as the average reduction in per-query BT loss (\Cref{eq:btloss}) on pair-wise model comparisons from unseen queries, relative to a single global BT model trained without clustering.
\textbf{(2) Hard partition structure analysis.}
To evaluate the structural properties of the hard cluster partitions (obtained by assigning each query to the cluster with the largest responsibility), we report \emph{Kendall-$\tau$} to measure the consistency among the BT profiles across clusters, as well as \emph{Adjusted Rand Index (ARI)} and \emph{Normalized Mutual Information (NMI)} to quantify the agreement between two different partitions.
Additional experimental details are in~\Cref{app:setup}.

\begin{figure*}[t]
    \centering
    \includegraphics[width=0.87\linewidth]{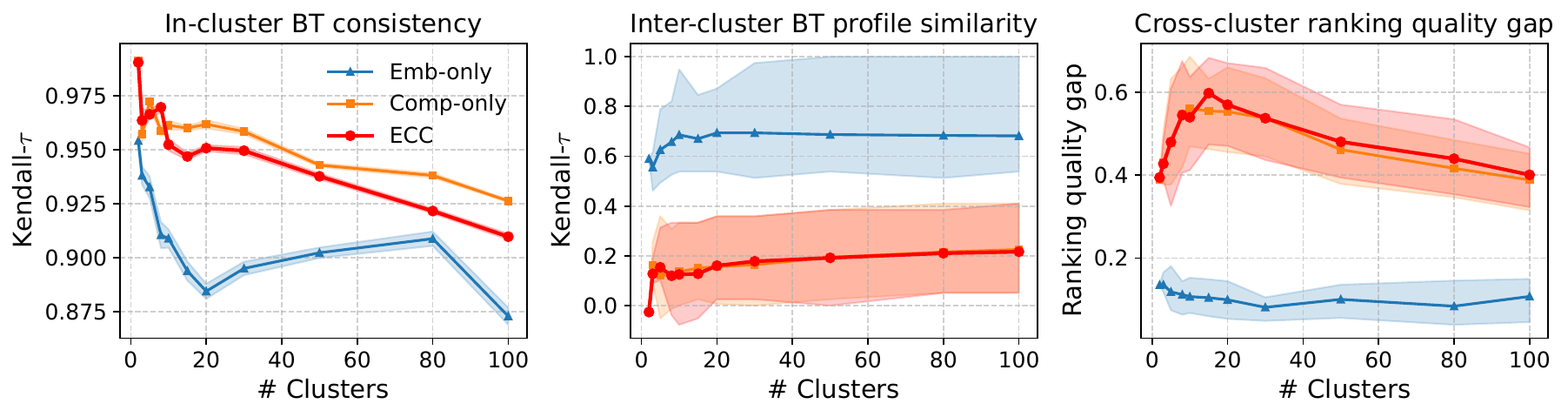}
    \caption{
    Consistency and separation of induced hard partitions on SPROUT, varying $\#$ of clusters.
    Left: in-cluster BT profile consistency ($\uparrow$);
    Middle: inter-cluster BT profile similarity ($\downarrow$);
    Right: cross-cluster ranking quality gap ($\uparrow$).
    Methods (\tsf{ECC}, Comp-only) using comparison signals has higher within-cluster consistency, lower inter-cluster similarity, and larger cross-cluster quality gaps.
    }
    \label{fig:consistency}
    \vskip -0.2in
\end{figure*}

\subsection{Main Results}
We compare different clustering methods across three inference signals, evaluating their performance on unseen queries with diverse capability demands, and address the following questions.

\noindent\textbf{RQ1: How does \tsf{ECC} compare to human-labeled clustering?}
We evaluate on three datasets with complete human-defined taxonomies:
(i) MMLU with 57 subject categories,
(ii) MMLU-Pro with 14 domain labels, and
(iii) MATH with 7 problem types.
For human-labeled clustering, we directly assign each query to its benchmark-provided ground-truth category and compute capability-aware ranking using the BT model trained within that category; thus, no additional test-time probe is needed for assignment.
For \tsf{ECC} and embedding-only clustering, we use Emb+Probe-infer to compute query-to-cluster responsibilities and induce mixture rankings.
For fairness, we set the number of clusters for \tsf{ECC} and embedding-only clustering to match the taxonomy granularity of each dataset.
As shown in Figure~\ref{fig:human}, \tsf{ECC} consistently achieves the highest capability-aware ranking quality gains on all three datasets, improving over human-labeled clustering by \textbf{17.64} pp and over embedding-only clustering by \textbf{17.99} pp on average.
Notably, despite using ground-truth semantic categories, human-defined taxonomies do not reliably reflect the underlying capability structure, yielding only marginal gains over the global BT model, and even degrading quality on MMLU.
These results underscore the limitations of both human-labeled and embedding-only clustering, and demonstrate that \tsf{ECC} effectively learns cluster structures aligned with the latent capability demands of queries, enabling more reliable capability inference for unseen queries.

\noindent\textbf{RQ2: How does the choice of clustering signals affect capability-aware ranking quality?}
To evaluate the impact of clustering signals, we compare \tsf{ECC}, embedding-only, and comparison-only clustering across benchmarks.
As shown in Figure~\ref{fig:main}, on SPROUT, \tsf{ECC} consistently achieves the best ranking quality across inference signals (additional results for other benchmarks are provided in~\Cref{fig:main_full}).
Compared to embedding-only clustering, Table~\ref{tab:main} shows that with $K{=}30$, \tsf{ECC} improves ranking quality gains by an average of \textbf{15.40} pp under Probe-infer, and \textbf{18.02} pp under Emb+Probe-infer across benchmarks.
Even when inference relies solely on embeddings (Emb-infer), \tsf{ECC} still outperforms embedding-only clustering.
The results highlight the limitations of relying on prior embeddings alone in clustering and the advantages of incorporating posterior comparison evidence during clustering.
Notably, comparison-only clustering performs well only when probe signals are available, but degrades significantly under Emb-infer. 
This sensitivity reveals a signal mismatch: clusters learned purely from comparison evidence do not provide a reliable capability inference when the inference signal relies on embeddings alone.
By integrating embeddings with comparison evidence, \tsf{ECC} produces a capability-aligned clustering structure that remains robust across inference signals, and outperforms comparison-only clustering under full-signal inference.
These results suggest that embeddings serve as a \emph{complementary} semantic prior to comparison evidence, enabling more accurate and robust alignment with diverse query capability demands.

\begin{figure*}[t]
    \centering
    \includegraphics[width=\linewidth]{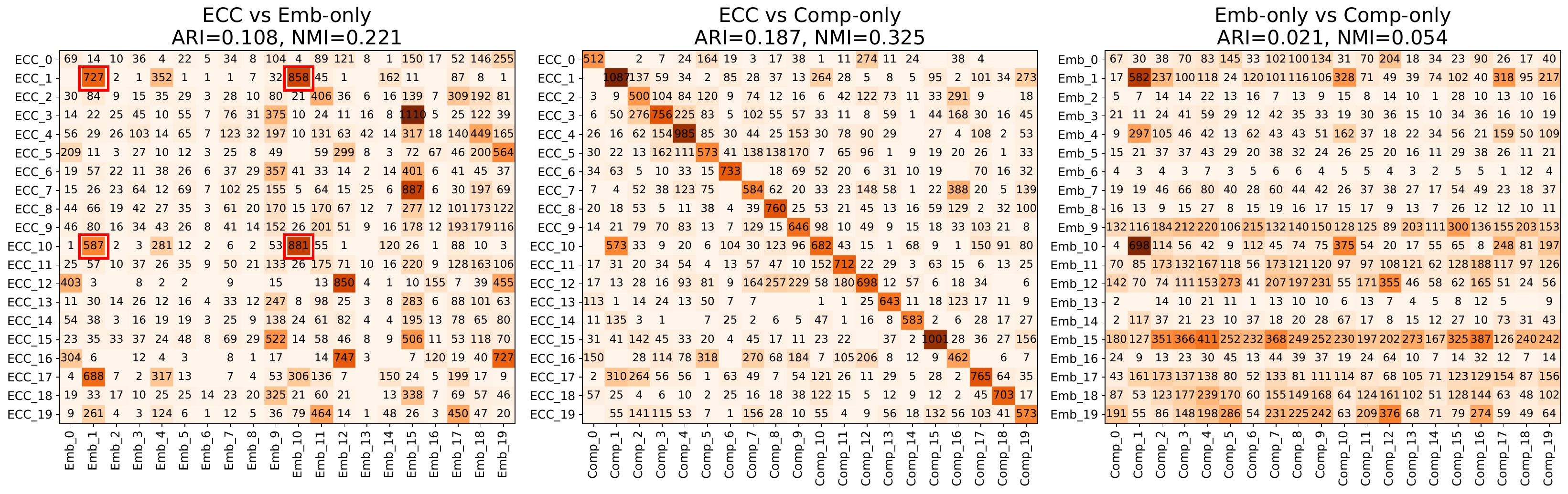}
    \caption{
        Cluster overlap matrix of hard partitions induced by three different clustering methods at $K=20$.
       Each cell shows the number of queries shared by a pair of clusters.
       ARI ($\uparrow$) and NMI ($\uparrow$) quantify overall agreement between the two clusterings.
    }
    \label{fig:matrix}
    \vskip -0.3in
\end{figure*}

\noindent\textbf{RQ3: How does the choice of inference signals affect capability-aware ranking quality?}
Holding the clustering method fixed and varying the inference signal, Figure~\ref{fig:main} shows that incorporating even a single probe comparison at inference time substantially improves capability-aware ranking quality. 
Across methods, the strongest performance is achieved by combining the probe with embeddings (Emb+Probe-infer), which consistently yields higher ranking quality gains than using embeddings alone or the probe alone. 
Table~\ref{tab:main} shows that for \tsf{ECC}, Emb+Probe-infer achieves the largest ranking quality gains, improving over Emb-infer by \textbf{19.06} pp and over Probe-infer by \textbf{5.64} pp.
This suggests that using both embeddings and a probe comparison yields a stronger query-specific signal, leading to more accurate capability inference of LLMs.
In contrast, embedding-only clustering exhibits only limited gains from adding the probe during inference, and in some cases, using the probe alone can even underperform Emb-infer.
This confirms the signal mismatch: clusters formed purely from embeddings are not necessarily aligned with the comparison evidence signals during capability inference.
As a result, probe evidence cannot be effectively exploited under embedding-only clustering, highlighting the importance of incorporating posterior comparison signals during clustering to enable robust and effective inference.

\noindent\textbf{RQ4: How do different clustering methods trade off within-cluster coherence and across-cluster separation in their hard partitions?}
We induce hard partitions by assigning each query to its highest-responsibility cluster and evaluate (i) within-cluster BT profile consistency via 5-fold Kendall-$\tau$, and (ii) cross-cluster separation via inter-cluster BT profile similarity and a cross-cluster ranking quality gap.
Figure~\ref{fig:consistency} shows that methods that incorporate comparison evidence during clustering (comparison-only clustering and \tsf{ECC}) consistently produce more coherent and better-separated partitions across a wide range of $K$: they achieve higher within-cluster consistency, lower inter-cluster similarity, and larger cross-cluster quality gaps.
In contrast, embedding-only clustering yields lower within-cluster consistency and weaker separation, suggesting that (i) clusters are less capability-consistent and tend to group queries with heterogeneous capability demands, resulting in unstable cluster-specific BT profiles, and (ii) queries requiring similar capabilities may be scattered across clusters, causing different clusters to induce similar BT profiles and providing limited capability-specific differentiation across LLMs.
Overall, these results underscore the importance of incorporating posterior comparison evidence during clustering, which helps calibrate the obtained cluster structure to be more coherent within clusters and more separable across clusters, better reflecting the underlying capability demands across queries.
Additional details and results for other benchmarks are provided in \Cref{app:rq5}.
We further analyze how the number of clusters $K$ impacts capability-aware ranking quality in \Cref{app:rq4} and include additional comparisons with P2L in~\Cref{app:p2l}.

\subsection{Cluster Overlap Analysis}
Figure~\ref{fig:matrix} presents the cluster overlap matrices between hard partitions at $K=20$.
Compared to embedding-only clustering, \tsf{ECC} exhibits concentrated overlaps: for most \tsf{ECC} clusters (rows), their queries predominantly fall into a small number of embedding-only clusters (columns).
This indicates that \tsf{ECC} preserves a coarse embedding-only-induced scaffold rather than arbitrarily reshuffling queries (ARI=$0.108$, NMI=$0.221$).
\tsf{ECC} also aligns more closely with comparison-only clustering (ARI=$0.187$, NMI=$0.325$), suggesting that the capability-relevant structure captured by posterior comparisons is retained in \tsf{ECC}.
Embedding-only and comparison-only clusterings show minimal agreement (ARI=$0.021$, NMI=$0.054$) and diffuse overlaps, indicating the misalignment between the two signals.
\tsf{ECC} helps bridge these signals by retaining semantic structure from embeddings while using comparison evidence to refine clusters in a more capability-aware structure. 
We further examine the internal structure of the induced clusters before and after incorporating posterior model-comparison signals.
A notable case appears in Figure~\ref{fig:matrix}: \texttt{ECC\_1}/\texttt{ECC\_10} and \texttt{Emb\_1}/\texttt{Emb\_10} form a prominent $2\times2$ block of four cells with large overlap counts.
This suggests that the two \tsf{ECC} clusters largely draw queries from the two embedding clusters (and vice versa), but with a systematically different partitioning.
To understand whether this repartitioning reflects semantic-topic shifts or deeper capability differences, we use GPT-5.2~\citep{openai2025introducinggpt52} as a standardized summarization aid to generate \textbf{human-verified post-hoc qualitative summaries} for the overlap cells, including semantic topics and implied capability demands.
These post-hoc summaries are used only for interpretation and are not used to define clusters or compute overlaps (full output in~\Cref{app:cluster}).

We first analyze \emph{within-embedding} splits to understand why posterior model comparisons separate queries that are close in embedding space. We fix the embedding cluster and compare the following intersection cells: (\texttt{ECC\_1}$\cap$\texttt{Emb\_1}) vs.\ (\texttt{ECC\_10}$\cap$\texttt{Emb\_1}), and (\texttt{ECC\_1}$\cap$\texttt{Emb\_10}) vs.\ (\texttt{ECC\_10}$\cap$\texttt{Emb\_10}).
Within \texttt{Emb\_1}, both cells fall into an advanced chemistry/biochemistry semantic region centered on mechanistic explanation and structure-property reasoning (e.g., ``mechanism of the reaction between benzene and bromine'') (all example prompts are in~\Cref{app:examples}).
However, posterior comparisons further separate this embedding cluster by {different capability demands}.
Queries in (\texttt{ECC\_1}$\cap$\texttt{Emb\_1}) tend to emphasize {evidence-grounded inference} from analytical or computational signals (e.g., ``using infrared spectroscopy and proton NMR spectroscopy data interpretation'').
In contrast, (\texttt{ECC\_10}$\cap$\texttt{Emb\_1}) focuses on {constraint-aware design} and {procedure planning} on top of mechanistic reasoning, introducing explicit optimization and safety/specificity constraints (e.g., ``specifically targets the cell wall of Gram-negative bacteria'') and requesting a concrete simulation plan (e.g., ``Use Molecular Dynamics simulations to investigate the interactions'').
A similar pattern appears within \texttt{Emb\_10} (see~\Cref{app:b2}), further indicating that within a fixed embedding-defined semantic topic, posterior model comparisons can separate queries with distinct capability demands, yielding a more capability-aligned partition.

We next analyze \emph{cross-embedding} merges to understand why semantically different embedding clusters can be merged under posterior model comparisons.
For a representative case, the two intersection cells within \texttt{ECC\_1} differ in semantics: (\texttt{ECC\_1}$\cap$\texttt{Emb\_1}) is framed in reaction/biochemistry terminology (e.g., ``concentration of the reactants affect the rate of the chemical reaction''), whereas (\texttt{ECC\_1}$\cap$\texttt{Emb\_10}) is framed in materials/device property control (e.g., ``optical and electrical properties of liquid crystals be tuned'').
However, posterior comparisons merge them into the same \texttt{ECC\_1} because both cells have the same capability demand: {parameter-to-outcome causal modeling with simulation-supported prediction}.
This shared capability signature is evidenced by the consistent reliance on computational reasoning in both (e.g., ``calculated using quantum chemistry methods'').
A similar pattern holds for \texttt{ECC\_10}, see~\Cref{app:b2} for details.
These results suggest that embeddings primarily separate queries by surface semantic topic, while posterior comparisons can \emph{merge cells across embeddings} when their capability demands align.

In summary, \tsf{ECC} \emph{splits} embedding-aligned queries when their latent capability demands diverge and \emph{merges} semantically different queries sharing aligned capability demands.
The resulting clusters are more capability coherent, making them better suited for LLM capability inference.
Visualization of the learned clusters under different clustering methods is in~\Cref{app:visual}.

\subsection{Applications and Additional Studies}

\noindent\textbf{Guided Optimal Query Routing.}
We consider an important downstream application of capability-aware ranking: \emph{optimal query routing}, where the goal is to route each query to the model that maximizes response quality.
For each test query, we infer a query-specific capability ranking over LLMs using a randomly sampled probe.
The query is routed to the top-ranked model, and we record its response score.
We report the average response quality over all test queries and compare against three baselines: (i) random routing, (ii) global routing to the top model from a single global BT ranking, and (iii) P2L~\citep{frick2025prompttoleaderboard}, which learns an LLM predictor to directly output query-conditioned LLM leaderboard. 
As shown in Figure~\ref{fig:application} (left), \tsf{ECC} achieves the highest average response quality on SPROUT across a wide range of $K$, outperforming all baselines.
This demonstrates that \tsf{ECC} learns a more capability-relevant clustering structure, which enables better prediction of the best-performing model and leads to improved routing performance. 
Full results are in~\Cref{app:qr}.

\begin{wrapfigure}[12]{R}{0.55\textwidth}
    \vspace{-1.3em}
    \centering
    \includegraphics[width=\linewidth]{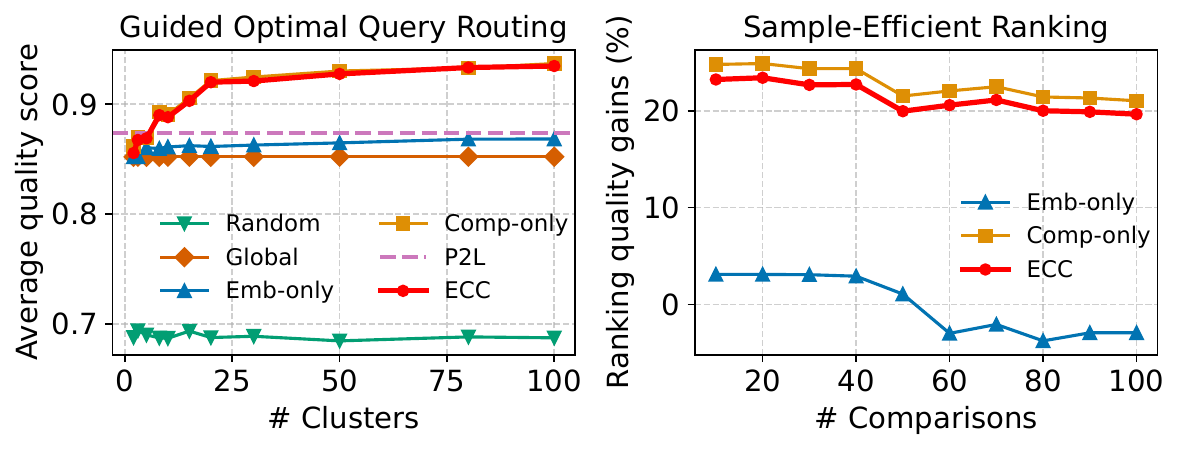}
    \caption{Performance of \tsf{ECC} on two applications: (i) guided optimal query routing (SPROUT) and (ii) sample-efficient new model ranking (RouterBench).}
    \label{fig:application}
\end{wrapfigure}

\noindent\textbf{Sample-Efficient New Model Ranking.}
We study another downstream task: inserting a new model into an existing ranking under a limited comparison budget.
We evaluate {sample efficiency} by varying the number of comparisons involving the new model, with one comparison per selected query.
Using only comparisons among existing models, each clustering method first clusters queries and then uses a fixed budget of new-model comparisons to estimate the ranking of the new model. 
We evaluate on held-out new-model comparisons and report the ranking quality gains relative to the single global BT baseline trained with the same old-model comparisons and the same new-model budget.
Figure~\ref{fig:application} (right) on RouterBench shows that \tsf{ECC} achieves substantial ranking quality gains under small budgets, consistently outperforming embedding-only clustering, and is comparable to comparison-only clustering, demonstrating its effectiveness in leveraging limited comparison evidence for new-model ranking.
Full results are in~\Cref{app:er}.

\noindent\textbf{Ablations, Additional Experiments, and Discussions.}
We ablate the number of model comparisons for clustering, trade-off parameter $\lambda$, clustering-query count, temperature $\T$, probe comparisons, embedding models, and tied comparisons, with details in~\Cref{app:abls}.
We further evaluate out-of-distribution generalization (\Cref{app:ood}) and robustness to different judges (\Cref{app:diff_judges}), and discuss evaluation judges and the evaluation efficiency of \tsf{ECC} in~\Cref{app:discuss}.

\section{Conclusion and Limitations}\label{sec:conclusion}
We present \tsf{ECC}, which calibrates prior semantic embeddings using limited posterior model comparisons to bridge the gap between surface-level semantics and latent capability alignment. 
By characterizing clusters via capability profiles and introducing trainable responsibilities, \tsf{ECC} accommodates queries with mixed capability demands, yielding a flexible, capability-aware clustering structure for {query-specific} LLM capability inference. 
We validate \tsf{ECC} through extensive quantitative and qualitative evaluations, demonstrating significant improvements in ranking quality and confirming its effectiveness in downstream tasks.

Our paper uses text-only queries and does not evaluate performance on multimodal inputs. 
However, \tsf{ECC} makes no assumptions about the embedding modality and can, in principle, accommodate multimodal query embeddings without changing its main clustering and inference procedures.
We also cluster queries using only the standalone text of each query, without considering potential conversational context or historical context. 
Extending our method to context-aware settings is an interesting direction for future work.


\newpage

\bibliographystyle{plainnat}  
\bibliography{src/references} 

\newpage

\appendix
\section{Experimental Details}\label{app:setup}
\noindent\textbf{Benchmarks.} 
We use three widely used benchmarks: 
i) SPROUT~\citep{somerstep2025carrot}, which covers 13 LLMs evaluated on a range of curated prompts;
ii) RouterBench~\citep{hu2024routerbench}, which includes the responses from 11 different LLMs over diverse domain-specific queries; 
iii) Open LLM Leaderboard v2~\citep{fourrier2024open}, collects the queries and corresponding responses from 16 LLMs.
These benchmarks span multiple datasets from diverse sources, including MMLU~\citep{hendrycks2020measuring}, MMLU-Pro~\citep{wang2024mmlu}, and MATH~\citep{hendrycks2021measuring}.
For each benchmark, all model responses are scored by the same powerful LLM-based evaluator to ensure consistency.
We split the queries into clustering and evaluation sets with varying proportions (ranging from 20\% to 90\% for clustering), using a default split of 80\% for clustering and 20\% for evaluation.
For each query, we decompose the full set of response scores into pairwise comparisons between models, using non-tied comparisons by default and considering tie-aware variants in the ablation study.
During clustering, we sample a varying number of comparison pairs (from 2 to 24), with 7 pairs per query used by default.
We vary the number of clusters from 2 to 100, using 30 as the default.
For each test query, unless otherwise specified, we hold out one model comparison as a probe for responsibility inference and evaluate on the remaining comparisons.
In the main setting, query embeddings are generated using \texttt{bge-base-en-v1.5}~\citep{bge_embedding}.
We additionally evaluate \texttt{SFR-Embedding-2\_R}~\citep{SFR-embedding-2} and
\texttt{text-embedding-3-large}~\citep{text-embedding-3-large}.
For both embedding and response generation, we use only the user query rather than the full input with the system prompt, avoiding contamination from additional system instructions and ensuring that embeddings reflect the query semantics.
Unless otherwise specified, we fix $\lambda$ as 3 and $\T$ as 0.1 for \tsf{ECC}.
For convergence, we set the threshold for the average KL divergence between successive responsibilities to $10^{-3}$, and the threshold for centroid stability to $10^{-4}$.

\noindent\textbf{Baselines.}
We consider two baseline clustering methods:
(i) \emph{Human-labeled clustering}, which applies hard cluster assignments based on static, predefined taxonomies from human annotations;
(ii) \emph{Embedding-only clustering}, which uses only the semantic embeddings of queries for clustering and follows the same soft-assignment update style as \tsf{ECC}.
Both baselines represent prior semantic clustering structures.
We additionally consider a variant of \tsf{ECC}, \emph{\textbf{Comparison-only (Comp-only)}}, which follows the same optimization procedure but excludes the embedding signal by setting $\lambda=0$.
For human-labeled clustering, we directly use the ground-truth labels as the cluster assignments for each query.
Since each benchmark aggregates multiple datasets, some of which lack human annotations, we restrict our comparison to representative datasets that are internally consistent, drawn from the same source, and fully labeled.
Specifically, we use three such datasets: (i) MMLU, (ii) MMLU-Pro, and (iii) MATH, all of which contain human-defined labels for every query.
For embedding-only clustering, we iteratively compute responsibilities via a temperature-controlled softmax over the embedding-centroid similarities, and update each centroid as the posterior-weighted mean $\bm c_k \leftarrow \mathrm{norm}\left(\sum_q r_{qk}\bm x_q / \sum_q r_{qk}\right)$. 
Furthermore, we use the same temperature parameter $\T$ as in \tsf{ECC}, and terminate the process when the maximum centroid shift falls below a fixed threshold.
When comparing clustering methods (Emb-only, Comp-only, and \tsf{ECC}), we use the same mixture ranking inference procedure as described in~\Cref{sec:method} for a fair comparison.

Additionally, we consider Prompt-to-Leaderboard (P2L)~\citep{frick2025prompttoleaderboard}, which trains an LLM-based predictor on queries and their pairwise model comparisons to predict query-conditioned BT coefficients and induce per-query leaderboards.
We implement P2L following the training objective and protocol described in~\citep{frick2025prompttoleaderboard}.
For each benchmark, we train \texttt{Qwen2.5-1.5B-Instruct}~\citep{qwen2.5} using 4 NVIDIA H200 GPUs, with a learning rate of $8\times10^{-6}$ and batch size of 4.
We use the same set of training queries as for other baselines, with each query associated with 7 pairwise comparisons, and evaluate on the same held-out queries and comparison pairs.

\noindent\textbf{Inference Signals.}
Given the learned clusters, we consider three inference signals for estimating the LLM capability ranking for a new query:
(i) \emph{Emb-infer}, which only uses the embedding signal of the query;
(ii) \emph{Probe-infer}, which only uses a probe model comparison of the query; 
(iii) \emph{Emb+Probe-infer}, which combines both signals.
We use the Emb+Probe-infer signal as the default inference method, unless otherwise specified.
Notably, all three inference settings follow the same posterior-weighted mixture ranking procedure described in~\Cref{sec:method}, where they differ only in the information used to infer the query-specific responsibilities.

\noindent\textbf{Metrics.}
We use two main metrics to evaluate different clustering methods:
\textbf{(1) Clustering performance.}
We evaluate the effectiveness of a capability clustering structure via its impact on \emph{capability-aware ranking quality} over unseen queries.
Specifically, we report the \emph{ranking quality gains}, defined as the average reduction in per-query BT loss (\Cref{eq:btloss}) on pair-wise model comparisons from unseen queries, relative to a single global BT model trained without clustering.
A larger improvement indicates better alignment between the clustering structure and the underlying capability distribution.
\textbf{(2) Hard partition structure analysis.}
To evaluate the structural properties of the hard cluster partitions (obtained by assigning each query to the cluster with the highest responsibility), we report \emph{Kendall-$\tau$} to measure the consistency among the BT profiles across different clusters, indicating whether the clustering leads to coherent and distinct capability orderings, as well as \emph{Adjusted Rand Index (ARI)} and \emph{Normalized Mutual Information (NMI)} to quantify the agreement between two different query partitions.
We also report the ranking quality gap, as the average absolute difference between per-query BT losses under two rankings evaluated over the same queries.

\section{Additional Experimental Results}

\begin{figure*}[t]
    \centering
    \includegraphics[width=\linewidth]{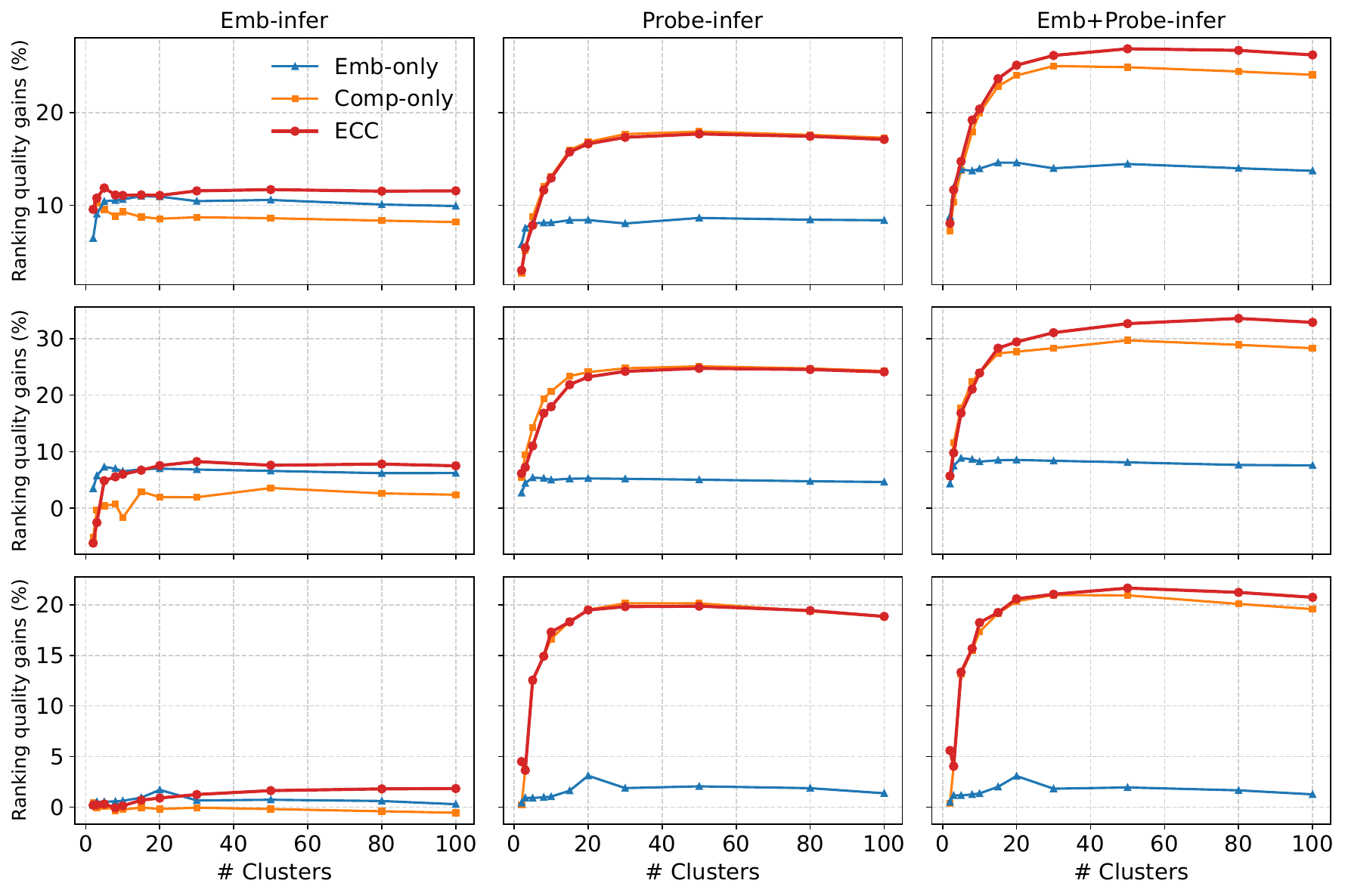}
    \caption{
   Ranking quality gains of clustering methods across all benchmarks, inference signals, and numbers of clusters $K$.
   \tsf{ECC} consistently achieves superior gains over embedding-only clustering.
    }
    \label{fig:main_full}
\end{figure*}

\begin{figure*}[t]
    \centering
    \includegraphics[width=\linewidth]{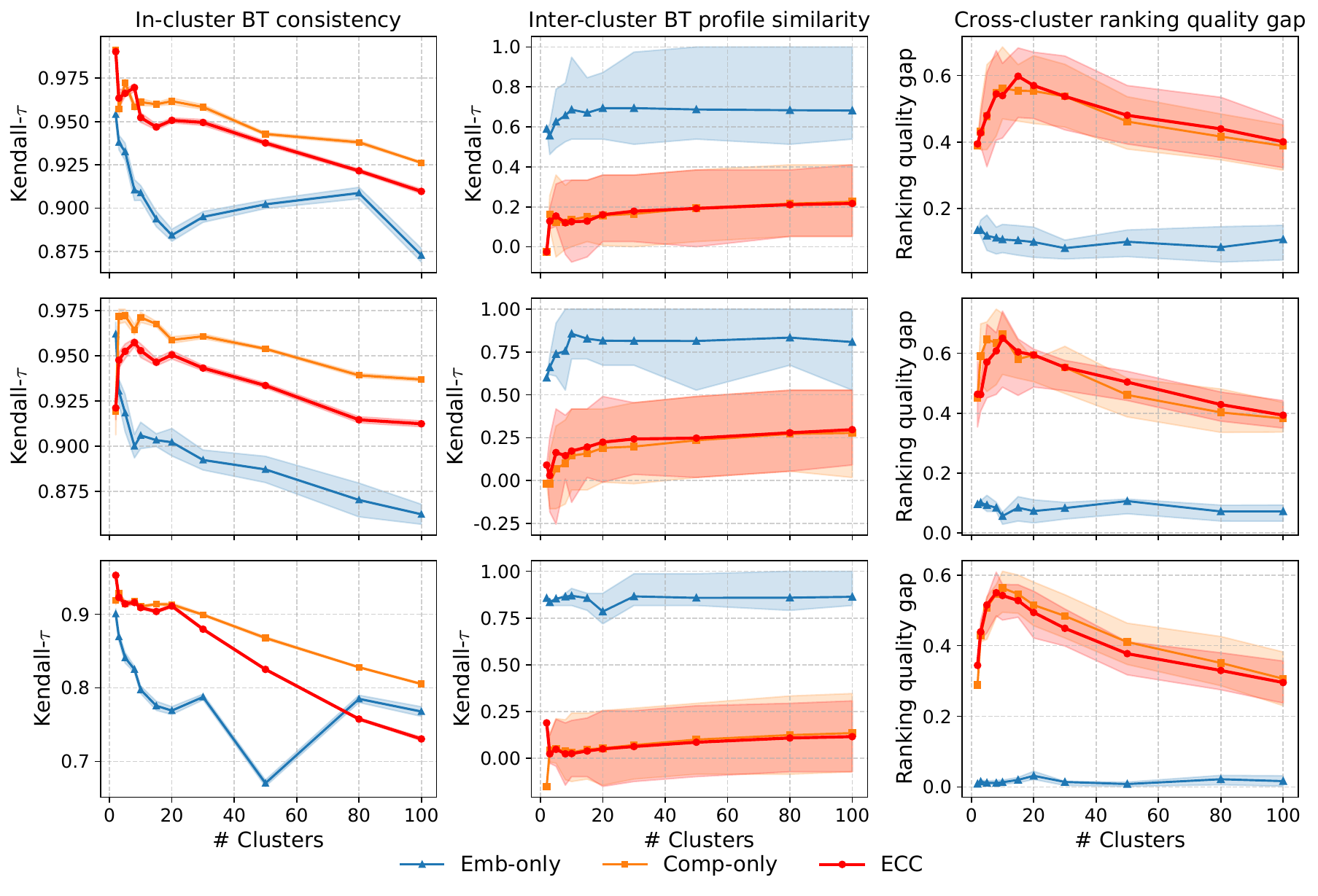}
    \caption{
    Consistency and separation of induced hard partitions across $K$ on three benchmarks.
    Left: in-cluster BT profile consistency ($\uparrow$);
    Middle: inter-cluster BT profile similarity ($\downarrow$);
    Right: cross-cluster ranking quality gap ($\uparrow$).
    Rows correspond to benchmarks: SPROUT (top), RouterBench (middle), and Open LLM Leaderboard v2 (bottom). 
    Methods incorporating comparison signals (\tsf{ECC}, Comp-only) consistently yield higher within-cluster consistency, lower inter-cluster similarity, and larger cross-cluster quality gaps.
    }
    \label{fig:consistency_full}
\end{figure*}

\begin{figure}[t]
    \centering
    \includegraphics[width=0.8\linewidth]{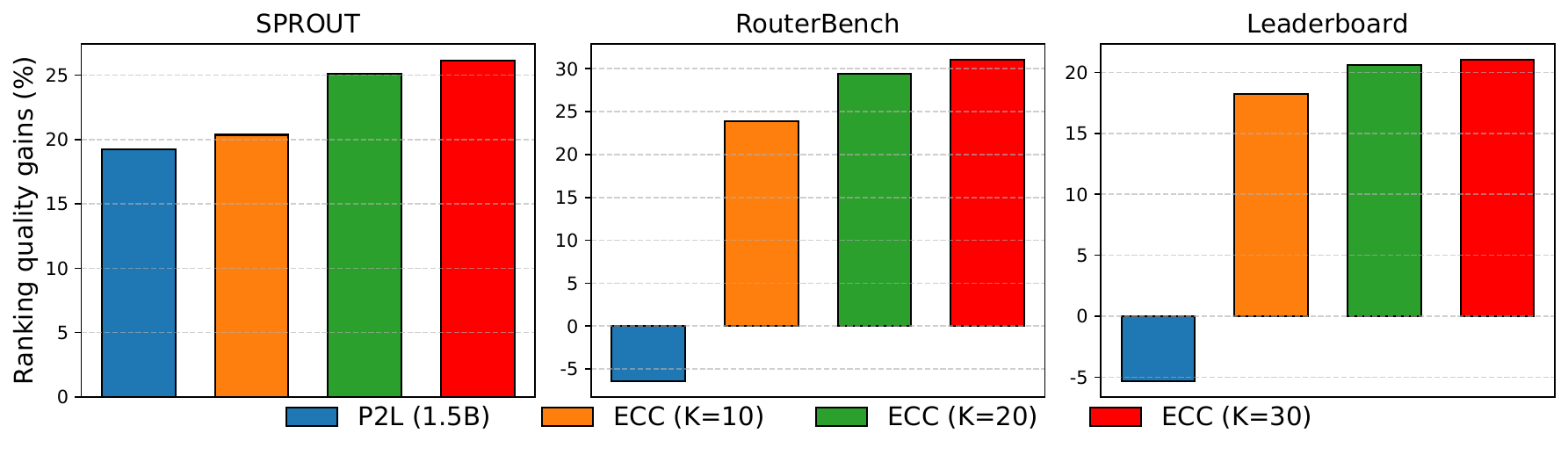}
    \caption{ 
    Ranking quality gains compared with P2L (1.5B) on three benchmarks. 
    \tsf{ECC} consistently outperforms P2L in query-specific ranking quality.
    }
    \label{fig:p2l_main}
\end{figure}

\subsection{RQ4: How do different clustering methods trade off within-cluster consistency and across-cluster separation in their induced hard partitions?}\label{app:rq5}
We first evaluate the within-cluster consistency of the induced hard partitions across clustering methods.
Specifically, for each hard cluster, we randomly split its queries into 5 folds, fit an independent BT model using only the queries in each fold, and measure agreement between the resulting BT profiles (rankings) by averaging pairwise Kendall-$\tau$ over all fold pairs. 
We repeat this procedure 10 times with different random splits and report an overall within-cluster consistency score by averaging across clusters.
As shown in Figure~\ref{fig:consistency_full}, (induced hard) clusters trained with comparison signals (comparison-only clustering and \tsf{ECC}) exhibit consistently higher within-cluster profile consistency across all benchmarks than embedding-only clustering. 
In contrast, embedding-only clustering yields lower within-cluster consistency, suggesting that the resulting clusters are less capability-coherent and tend to mix queries with heterogeneous capability demands, thereby making the cluster-specific BT profiles less stable across query subsets.

To further analyze cluster structure and separability, we compute two cross-cluster measures. 
First, for each $K$, we compute Kendall-$\tau$ between the BT profiles of every pair of hard clusters, and report the mean as the inter-cluster BT profile similarity with the shaded region showing the 25th–75th percentiles of these pairwise similarities.
Second, for each target hard cluster, we compute a cross-cluster ranking quality gap by evaluating queries of that cluster under (i) BT models trained on all other clusters (and averaging the resulting ranking qualities) and (ii) the BT model trained on the same cluster. 
The gap is the difference between these two qualities, with larger values indicating better separation. 
For each $K$, we report the mean gap across clusters, with the shaded region showing the 25th–75th percentiles of the gaps across clusters.
Figure~\ref{fig:consistency} shows that embedding-only clustering produces substantially higher inter-cluster BT profile similarity and smaller ranking quality gaps across $K$, indicating a weak separation, where queries that require similar underlying capabilities may be scattered across multiple clusters, with different clusters end up inducing similar BT profiles, thus providing limited improvement in distinguishing and inferring capability-specific performance across LLMs. 
In contrast, comparison-only clustering and \tsf{ECC} yield both lower inter-cluster similarity and larger ranking quality gaps across all benchmarks, suggesting that the induced partitions are better separated in capability space.

Overall, these results underscore the importance of incorporating posterior model-comparison evidence during clustering, which helps calibrate the obtained cluster structure to be more coherent within clusters and more separable across clusters, better reflecting the underlying capability distribution across queries.

\subsection{RQ5: How does increasing the number of clusters $K$ affect capability-aware ranking quality?}\label{app:rq4}
When varying $K$ from 2 to 100, Figure~\ref{fig:main_full} reveals a consistent scaling trend across benchmarks.
Under Probe-infer and Emb+Probe-infer, ranking quality gains increase rapidly for small $K$ and then saturate, with slight declines at very large $K$. 
In contrast, under Emb-infer, performance is relatively insensitive to $K$, indicating limited benefit from increasing cluster granularity when inference relies only on embeddings.
Across all $K$ and inference signals, \tsf{ECC} remains the best-performing method. 
However, extremely large $K$ leads to a slight performance drop due to noisier cluster-specific BT estimates, as increasing $K$ spreads the responsibilities mass across more clusters, reducing the effective comparisons supporting each cluster.
Overall, \emph{a moderate number of clusters} ($K$=20–50) suffices to capture the key capability variations among queries.

\subsection{Comparison with P2L}\label{app:p2l}
We compare \tsf{ECC} with Prompt-to-Leaderboard (P2L)~\citep{frick2025prompttoleaderboard}, a recent approach that trains an LLM-based predictor to directly output query-conditioned BT coefficients and induce per-query leaderboards from human-judged comparisons.
We implement P2L following the same training objective and process as used in~\citep{frick2025prompttoleaderboard}, using \texttt{Qwen2.5-1.5B-Instruct}~\citep{qwen2.5} on each benchmark.
As shown in Figure~\ref{fig:p2l_main}, \tsf{ECC} consistently achieves higher capability-aware ranking quality gains on unseen queries than P2L, across all three benchmarks and across different numbers of clusters.
Notably, on both RouterBench and Leaderboard, P2L even underperforms the global BT baseline, resulting in negative ranking quality gains.
These results suggest that directly learning an end-to-end query-to-ranking predictor can be sample-inefficient and sensitive to the amount and distribution of training data, especially when surface semantic signals are not well aligned with comparison-derived capability signals.
Under limited supervision, P2L may overfit surface semantic patterns rather than recover the latent capability structure needed for query-specific capability inference.
As a result, its predictions on unseen queries can be unstable and may even underperform the global BT baseline.
In contrast, \tsf{ECC} explicitly learns reusable capability profiles from comparison evidence and performs query-specific inference through lightweight responsibility estimation, enabling more reliable rankings without training an additional LLM.

\begin{figure*}[t]
  \centering
  \begin{subfigure}[t]{0.31\linewidth}
    \centering
    \includegraphics[width=\linewidth]{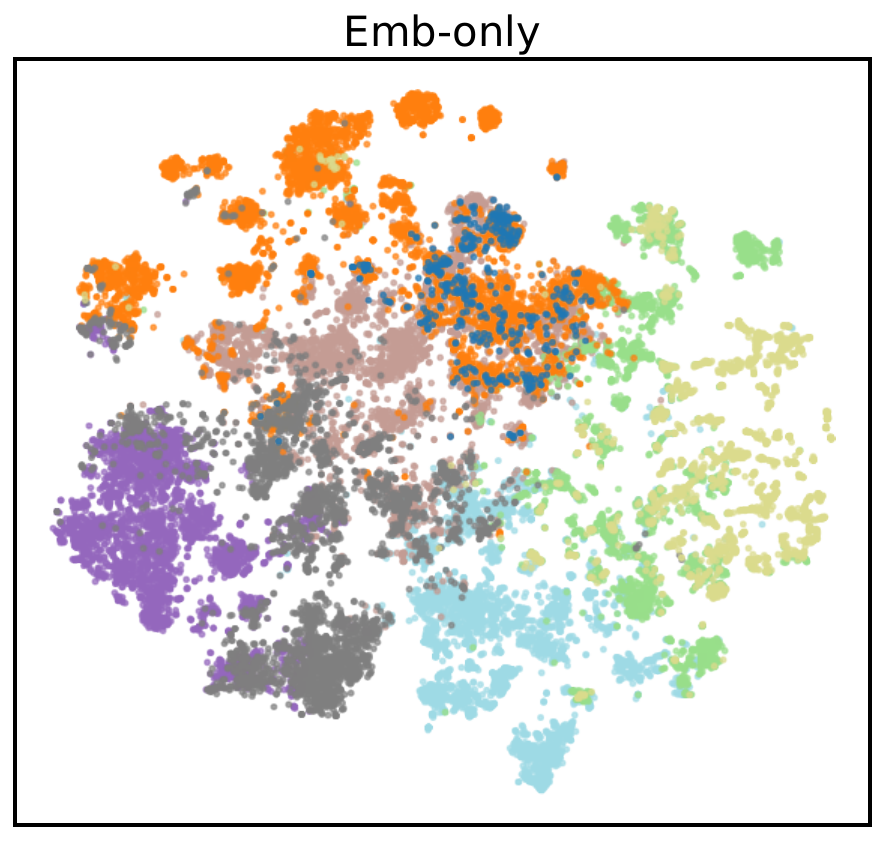}
    \label{fig:visuala}
  \end{subfigure}
  \begin{subfigure}[t]{0.31\linewidth}
    \centering
    \includegraphics[width=\linewidth]{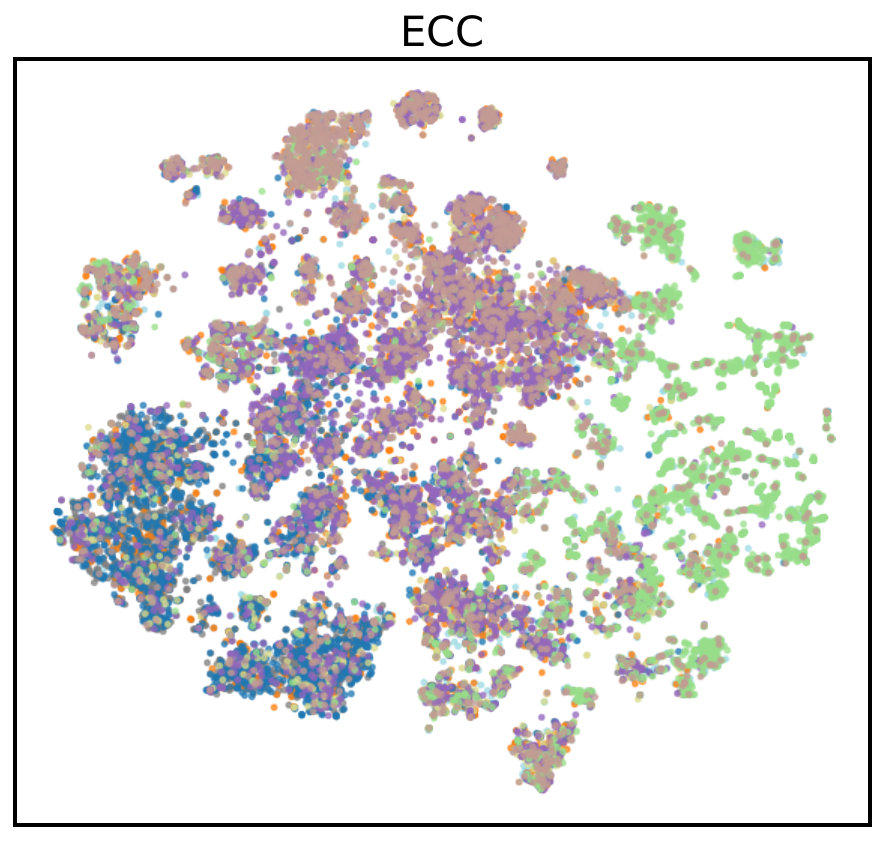}
    \label{fig:visualb}
  \end{subfigure}
  \begin{subfigure}[t]{0.31\linewidth}
    \centering
    \includegraphics[width=\linewidth]{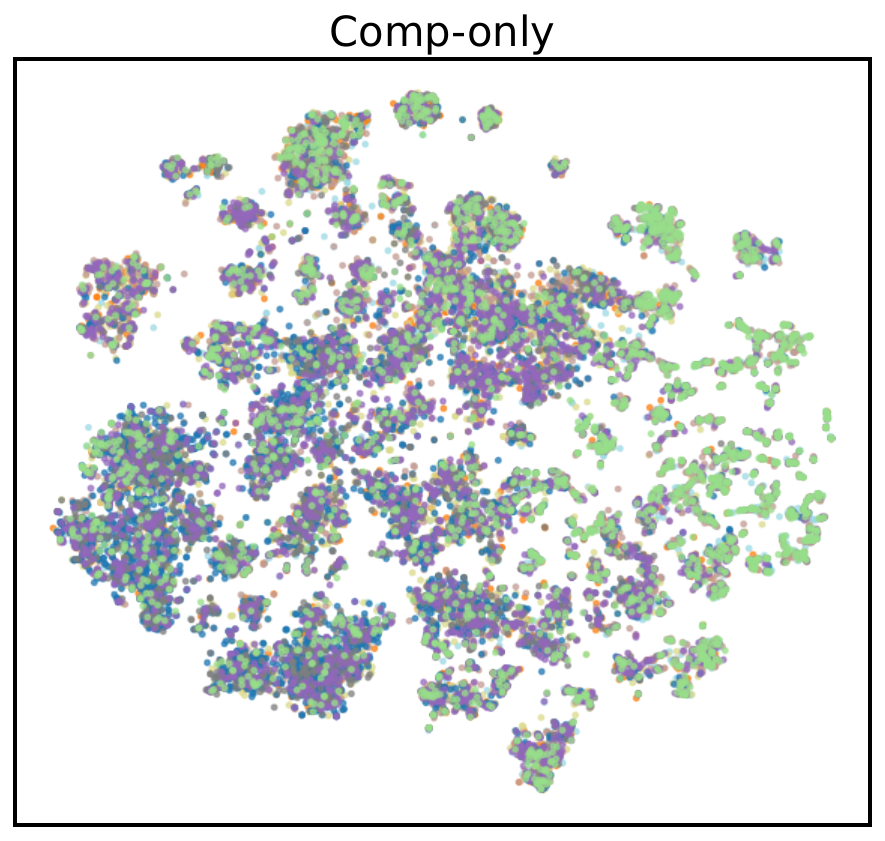}
    \label{fig:visualc}
  \end{subfigure}
  \caption{t-SNE visualization of hard partitions from three clustering methods with $K=8$.}
  \label{fig:visual}
\end{figure*}

\begin{figure}[t]
    \centering
    \includegraphics[width=\linewidth]{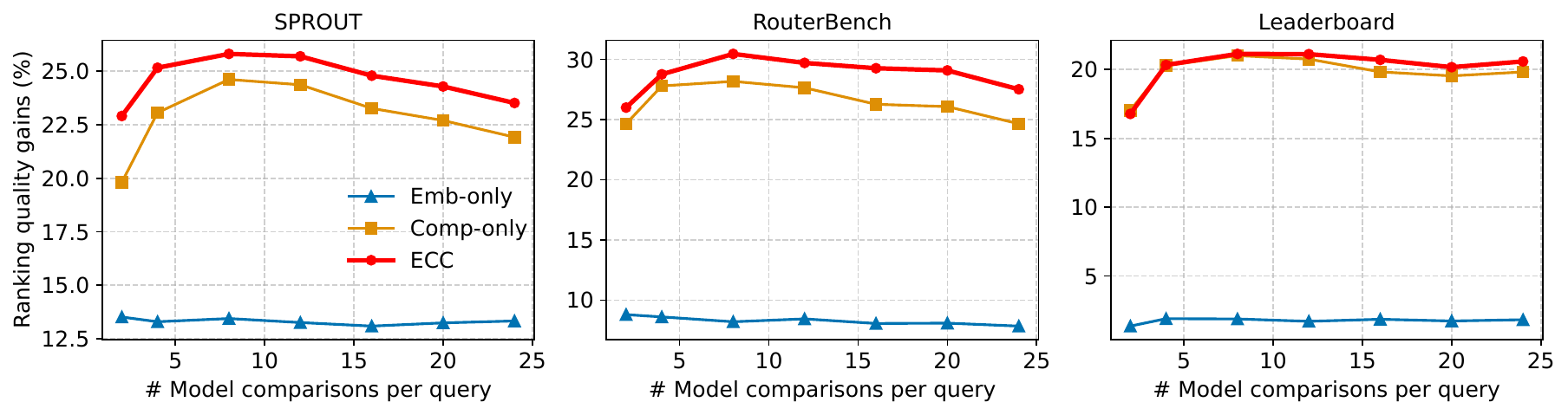}
    \caption{
    Ablation studies on the number of model comparisons per query used during clustering across three evaluation benchmarks.
    }
    \label{fig:abl_pairs}
\end{figure}

\begin{figure}[t]
    \centering
    \includegraphics[width=\linewidth]{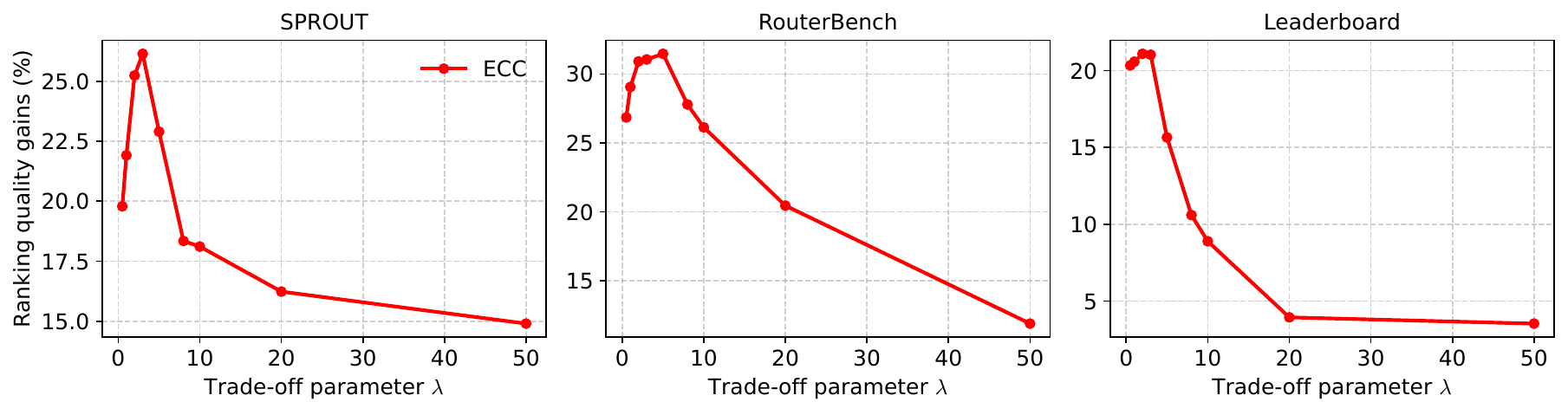}
    \caption{
    Ablation studies on the trade-off parameter $\lambda$ across three evaluation benchmarks.
    }
    \label{fig:abl_lambda}
\end{figure}

\subsection{Cluster Overlap Analysis}\label{app:b2}
Figure~\ref{fig:matrix} presents the pairwise cluster overlap matrices of hard partitions induced by three clustering methods at $K{=}20$.
Compared to embedding-only clustering, \tsf{ECC} exhibits clearly concentrated overlaps: for a majority of \tsf{ECC} clusters (rows), their queries predominantly fall into a small number of embedding-only clusters (columns).
This indicates that \tsf{ECC} preserves a coarse embedding-only-induced scaffold rather than arbitrarily reshuffling queries (ARI=$0.108$, NMI=$0.221$).
Furthermore, \tsf{ECC} aligns more strongly with comparison-only clustering (ARI=$0.187$, NMI=$0.325$), suggesting that the capability-relevant structure captured by posterior comparisons is also retained in \tsf{ECC}.
In contrast, embedding-only and comparison-only clusterings show substantially weaker agreement (ARI=$0.021$, NMI=$0.054$) and more diffuse overlaps, highlighting that embeddings and comparison signals are not naturally aligned and capture different structures.
\tsf{ECC} helps bridge these two views by retaining semantic organization from embeddings while using model comparison evidence to refine clusters in a better capability-aware structure, consistent with the improved ranking quality reported earlier.

We further examine the internal structure of the induced clusters before and after incorporating posterior model-comparison signals.
When comparing \tsf{ECC} to embedding-only clustering, we observe a clear block-wise overlap in Figure~\ref{fig:matrix}: the overlaps between \texttt{ECC\_1}/\texttt{ECC\_10} and \texttt{Emb\_1}/\texttt{Emb\_10} form a prominent $2\times2$ block of four cells, which contains the substantial overlaps between these four clusters across the two partitions.
This indicates that the two \tsf{ECC} clusters largely draw queries from the two embedding clusters (and vice versa), but with a systematically different partitioning.
To understand whether this repartitioning reflects semantic-topic changes or deeper capability differences, we conduct a qualitative comparison of the prompts and their implied capability demands, using GPT-5.2~\citep{openai2025introducinggpt52} as a standardized summarization aid to support consistent comparison, with all summaries human-verified (full analysis in~\Cref{app:cluster}).

We first analyze \emph{within-embedding} splits to understand why posterior model comparisons separate queries that are close in embedding space.
Concretely, we fix the embedding cluster and compare how \tsf{ECC} repartitions it by contrasting the two \tsf{ECC} intersection cells within each embedding cluster: (\texttt{ECC\_1}$\cap$\texttt{Emb\_1}) vs.\ (\texttt{ECC\_10}$\cap$\texttt{Emb\_1}), and (\texttt{ECC\_1}$\cap$\texttt{Emb\_10}) vs.\ (\texttt{ECC\_10}$\cap$\texttt{Emb\_10}).
Within \texttt{Emb\_1}, both cells fall into an advanced chemistry/biochemistry semantic region centered on mechanistic explanation and structure-property reasoning (e.g., ``mechanism of the reaction between benzene and bromine'' and ``mechanism of ligand exchange in metal-organic frameworks (MOFs)'').
However, posterior comparisons further separate this embedding cluster by {different capability demands}.
Queries assigned to (\texttt{ECC\_1}$\cap$\texttt{Emb\_1}) more consistently foreground \emph{evidence-grounded inference} from analytical or computational signals (e.g., ``using infrared spectroscopy and proton NMR spectroscopy data interpretation'' and ``based on Density Functional Theory calculations'').
In contrast, (\texttt{ECC\_10}$\cap$\texttt{Emb\_1}) more often emphasizes \emph{constraint-aware design} and \emph{procedure planning} on top of mechanistic reasoning, introducing explicit optimization and safety/specificity constraints (e.g., ``specifically targets the cell wall of Gram-negative bacteria'' and ``without harming healthy cells in the body'') and requesting a concrete simulation plan (e.g., ``Use Molecular Dynamics simulations to investigate the interactions'').

A similar phenomenon appears within \texttt{Emb\_10}, where both cells lie in a materials/polymer/nanomaterials topic concerned with how controllable design parameters affect functional properties (e.g., ``dopant concentration affect these properties'' and ``optimized to achieve the highest absorption capacity'').
However, the two \tsf{ECC} cells emphasize different capability requirements.
(\texttt{ECC\_1}$\cap$\texttt{Emb\_10}) places a stronger weight on \emph{physics-based prediction and mechanistic sensitivity}, often invoking first-principles or simulation-based reasoning (e.g., ``density functional theory (DFT)'' and ``Using molecular dynamics simulations, investigate'') and focusing on nanoscale structure/property levers (e.g., ``variations in their size, shape, and composition'').
In contrast, (\texttt{ECC\_10}$\cap$\texttt{Emb\_10}) more often emphasizes \emph{process- and application-driven optimization}, focusing on actionable tuning of manufacturing/processing parameters (e.g., ``processing parameters or equipment can be made'') and requests step-by-step characterization procedures (e.g., ``step-by-step explanation of the DMA technique'' and ``determine the molecular weight distribution'').
Taken together, these splits suggest that even under a fixed embedding-defined semantic topic, posterior model comparisons can distinguish prompts that reveal heterogeneous capability requirements and refine the cluster structure by repartitioning the neighborhood into more capability-aligned clusters.

We next analyze \emph{cross-embedding} merges to understand why semantically different embedding clusters can be merged under posterior model comparisons.
For a representative case, the two intersection cells within \texttt{ECC\_1} differ in semantics: (\texttt{ECC\_1}$\cap$\texttt{Emb\_1}) is framed in reaction/biochemistry terminology (e.g., ``concentration of the reactants affect the rate of the chemical reaction''; ``reaction mechanism for the conversion of an alcohol to an alkene''), whereas (\texttt{ECC\_1}$\cap$\texttt{Emb\_10}) is framed in materials/device property control (e.g., ``optical and electrical properties of liquid crystals be tuned''; ``size of a quantum dot affect its electronic and optical properties'').
However, posterior comparisons merge them into the same \texttt{ECC\_1} because both cells instantiate the same capability demand: \emph{parameter-to-outcome causal modeling with simulation-supported prediction}.
This shared capability signature is evidenced by the consistent reliance on computational reasoning in both cells (e.g., ``calculated using quantum chemistry methods'' vs.\ ``density functional theory calculations'' / ``Monte Carlo simulations'').

\begin{figure}[t]
    \centering
    \includegraphics[width=\linewidth]{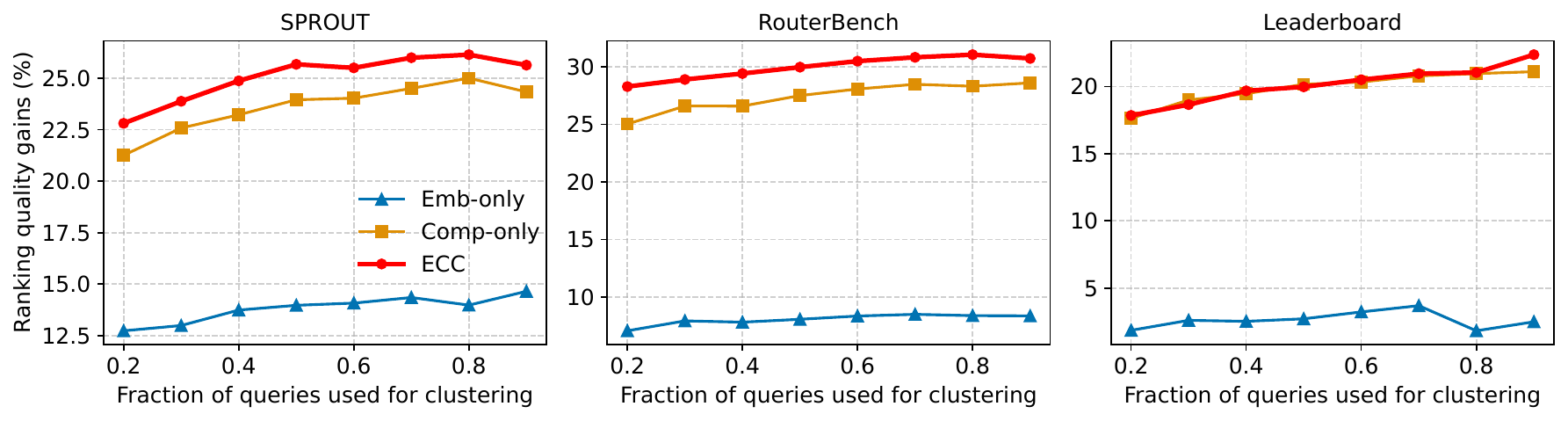}
    \caption{
    Ablation on the number of clustering queries across three benchmarks. 
    }
    \label{fig:abl_query}
\end{figure}

\begin{figure}[t]
    \centering
    \includegraphics[width=\linewidth]{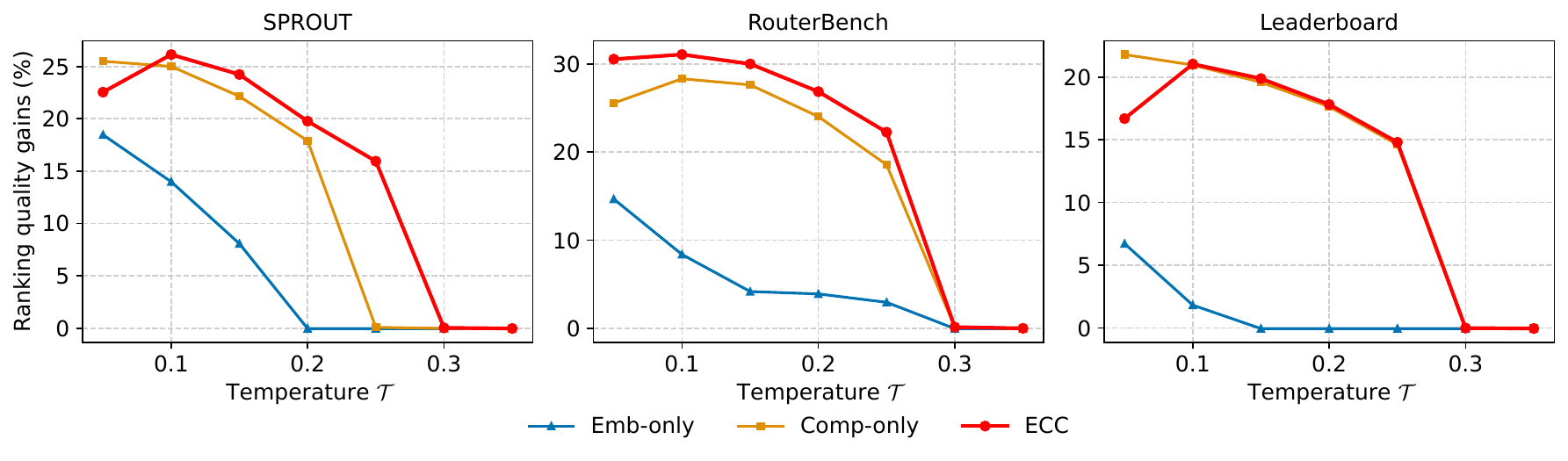}
    \caption{
    Ablation on the temperature $\T$ across three benchmarks. 
    }
    \label{fig:abl_temp}
\end{figure}

A similar pattern holds for \texttt{ECC\_10}: (\texttt{ECC\_10}$\cap$\texttt{Emb\_1}) emphasizes biomedical/chemical decision making under constraints (e.g., ``selectively targets cancer cells while sparing healthy cells''; ``without affecting the related enzyme Cyclooxygenase-1 (COX-1)''), while (\texttt{ECC\_10}$\cap$\texttt{Emb\_10}) emphasizes materials/process optimization (e.g., ``degree of crystallinity of a polymer affect its mechanical properties''; ``optimize the printing conditions for a specific polymer-based material''; ``optimized while maintaining high selectivity'').
Despite the semantic shift, both cells demand an aligned capability, \emph{constrained optimization with model-based justification}, realized via analytical formulations (e.g., ``using a rate equation for the reaction in question'') or simulation workflows (e.g., ``Use molecular dynamics simulations to investigate''; ``Monte Carlo simulation'').
These results suggest that embeddings primarily separate queries by surface semantic topic (\texttt{Emb\_1} vs.\ \texttt{Emb\_10}), while posterior model comparisons can \emph{merge cells across embeddings} when their capability demands align.

Taken together, this $2\times2$ block reveals a consistent pattern: \tsf{ECC} \emph{splits} embedding-aligned queries when their latent capability demands diverge, and \emph{merges} embedding-separated (semantically different) queries when they share aligned capability demands.
The resulting clusters are therefore more capability coherent, making them better suited for capability inference of LLMs.

\subsection{Clustering Visualization}\label{app:visual}
Figure~\ref{fig:visual} visualizes the hard partitions produced by different clustering methods with $K=8$.
Compared to embedding-only clustering, clustering based solely on model comparisons is substantially less aligned with the embedding geometry, resulting in visibly mixed clusters that group together queries distant in embedding space.
In contrast, \tsf{ECC} preserves much of the coarse embedding-based organization while still adjusting boundaries based on model comparison signals.
This view echoes our qualitative analysis and quantitative results: embeddings provide a useful semantic scaffold, while posterior model comparison signals help reshape the clustering structure to be more capability-aware.

\begin{figure}[t]
    \centering
    \includegraphics[width=\linewidth]{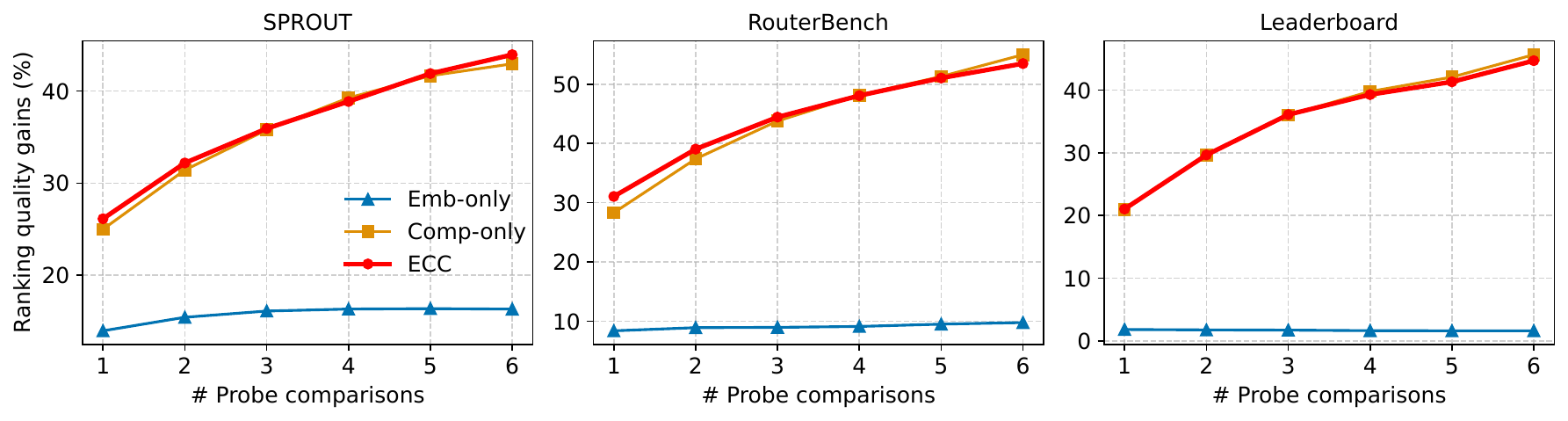}
    \caption{
    Ablation study on the number of probe model comparisons used for evaluation across three benchmarks.
    }
    \label{fig:probe}
\end{figure}

\begin{figure}[t]
    \centering
    \includegraphics[width=\linewidth]{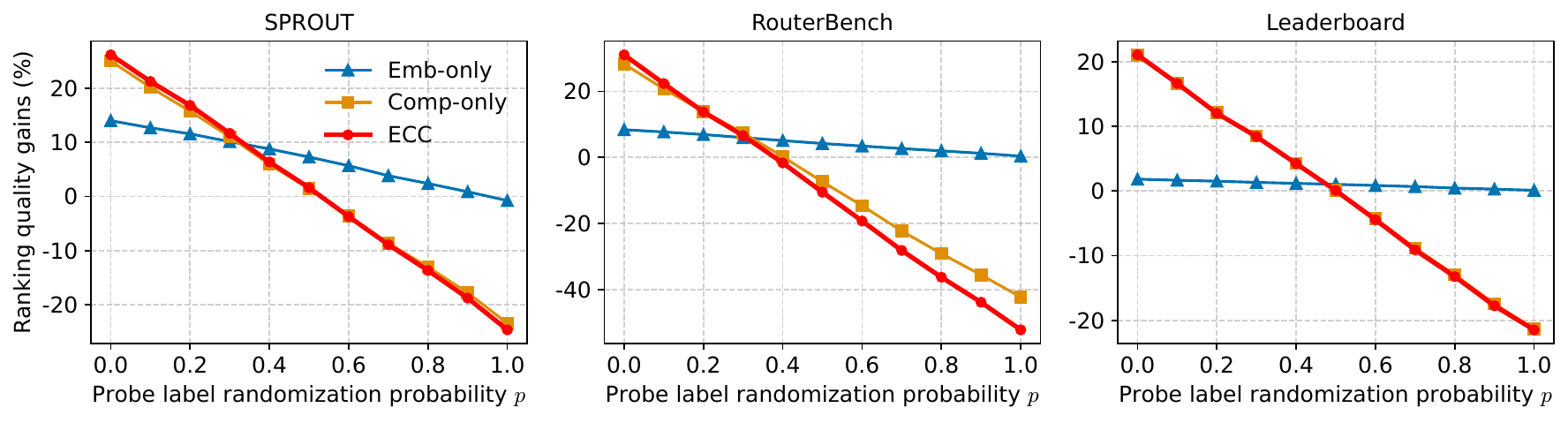}
    \caption{
    Ablation on probe-label randomization across three benchmarks.
    }
    \label{fig:probe_random}
\end{figure}

\begin{table*}[t]
\centering
\small
\caption{
Robustness of \tsf{ECC} to probe selection at inference time across three benchmarks.
We report the mean $\pm$ std of ranking quality gains (\%) over different probe selections for each clustering method.
}
\begin{tabular}{lccc}
\toprule
{\textbf{{Clustering Method}}} 
& {\textbf{SPROUT}} 
& {\textbf{RouterBench}} 
& {\textbf{Open LLM Leaderboard v2}} \\
\midrule
{Emb-only} & 13.99 $\pm$ 0.12 & 8.48 $\pm$ 0.12 & 1.79 $\pm$ 0.05   \\
{Comp-only} & 24.92 $\pm$ 0.25 & 28.21 $\pm$ 0.56 & 20.92 $\pm$ 0.29 \\
\textbf{\tsf{ECC}} & \textbf{26.41 $\pm$ 0.31} & \textbf{30.95 $\pm$ 0.50} & \textbf{21.21 $\pm$ 0.21}   \\
\bottomrule
\end{tabular}
\label{tab:var_probe}
\end{table*}

\begin{figure}[t]
    \centering
    \includegraphics[width=0.7\linewidth]{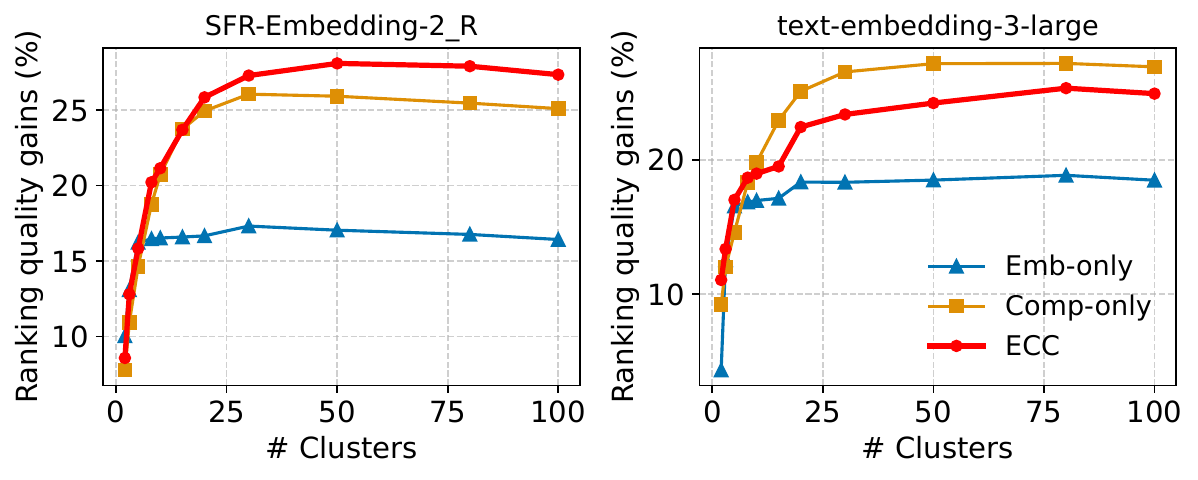}
    \caption{
    Ablation of embedding models on SPROUT.
    }
    \label{fig:emb}
\end{figure}

\begin{figure}[t]
    \centering
    \includegraphics[width=\linewidth]{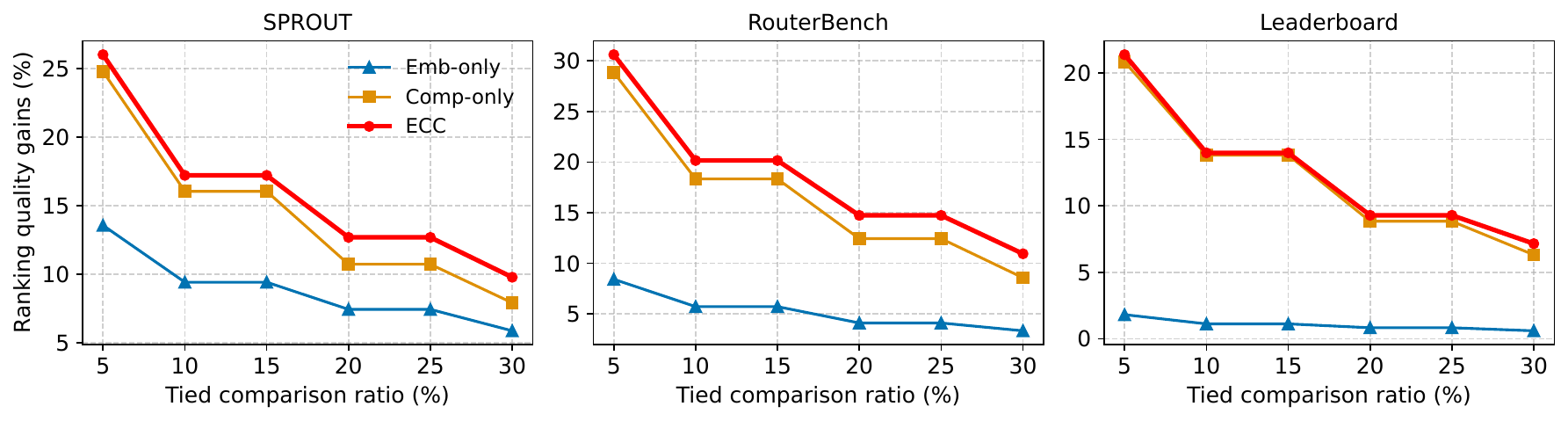}
    \caption{
    Robustness to tied comparisons.
    Across varying tied-comparison ratios, \tsf{ECC} consistently achieves the best or comparable ranking quality gains across benchmarks, showing robustness to tie-aware comparison settings.
    }
    \label{fig:tie}
\end{figure}

\begin{table}[t]
\centering
\small
\caption{Ranking quality gains (\%) of \tsf{ECC} and embedding-only clustering in the out-of-distribution query setting.}
\setlength{\tabcolsep}{5pt}
\begin{tabular}{lcccccccc}
\toprule
\textbf{Method} & \textbf{$K=5$} & \textbf{$K=10$} & \textbf{$K=15$} & \textbf{$K=20$} & \textbf{$K=30$} & \textbf{$K=50$} & \textbf{$K=80$} & \textbf{$K=100$} \\
\midrule
Emb-only & 2.79 & 2.96 & 2.93 & 2.71 & 2.48 & 1.86 & 0.86 & 0.13 \\
\tsf{ECC} & \textbf{8.19} & \textbf{17.41} & \textbf{19.18} & \textbf{19.70} & \textbf{20.77} & \textbf{20.20} & \textbf{17.75} & \textbf{15.94} \\
\bottomrule
\end{tabular}
\label{tab:ood}
\end{table}

\begin{table}[t]
\centering
\small
\caption{Ranking quality gains (\%) of \tsf{ECC} and embedding-only clustering on MMLU-Pro under different judges.}
\setlength{\tabcolsep}{4.4pt}
\scalebox{0.9}
{
\begin{tabular}{lcccccccc}
\toprule
\textbf{Method} & \textbf{$K=5$} & \textbf{$K=10$} & \textbf{$K=15$} & \textbf{$K=20$} & \textbf{$K=30$} & \textbf{$K=50$} & \textbf{$K=80$} & \textbf{$K=100$} \\
\midrule
Emb-only (SPROUT) & 3.16 & 3.65 & 3.48 & 3.30 & 3.36 & 2.75 & 1.99 & 1.54 \\
\tsf{ECC} (SPROUT) & \textbf{15.09} & \textbf{20.96} & \textbf{21.54} & \textbf{21.46} & \textbf{22.01} & \textbf{21.70} & \textbf{19.86} & \textbf{18.69} \\
\midrule
Emb-only (Leaderboard) & 1.63 & 2.07 & 2.19 & 2.05 & 1.61 & 0.78 & -0.57 & -1.45 \\
\tsf{ECC} (Leaderboard) & \textbf{10.31} & \textbf{17.09} & \textbf{18.55} & \textbf{19.14} & \textbf{19.55} & \textbf{18.41} & \textbf{15.78} & \textbf{14.08} \\
\bottomrule
\end{tabular}
}
\label{tab:diff_judges}
\end{table}

\begin{figure}[t]
    \centering
    \includegraphics[width=\linewidth]{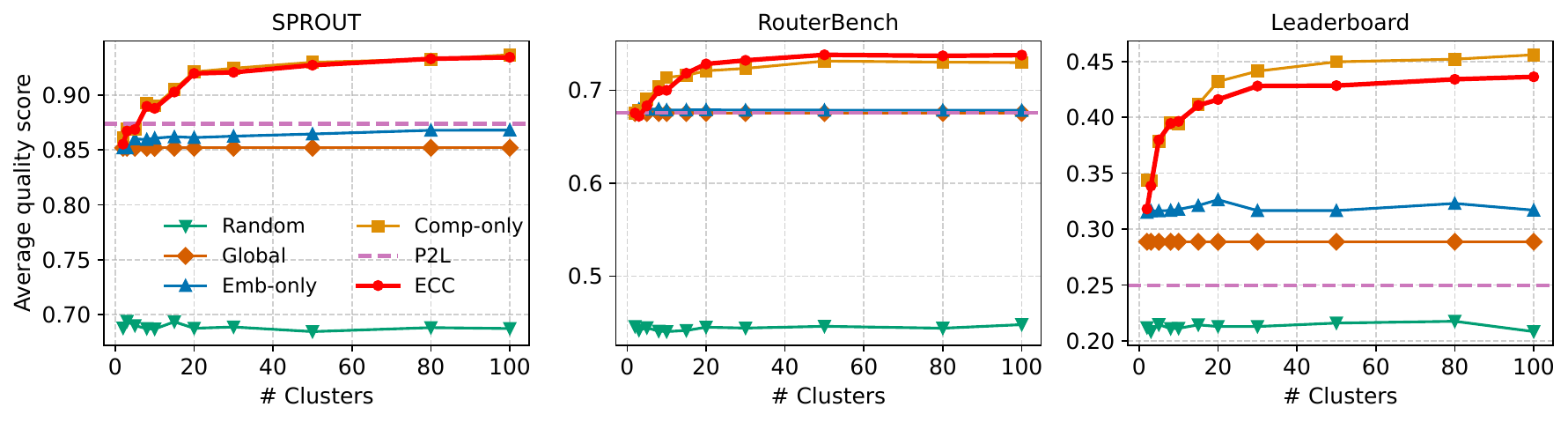}
    \caption{
    Performance of \tsf{ECC} on the downstream task of guided query routing across three benchmarks, where it consistently outperforms baselines by achieving higher response quality per query.
    }
    \label{fig:routing}
\end{figure}

\begin{figure}[t]
    \centering
    \includegraphics[width=\linewidth]{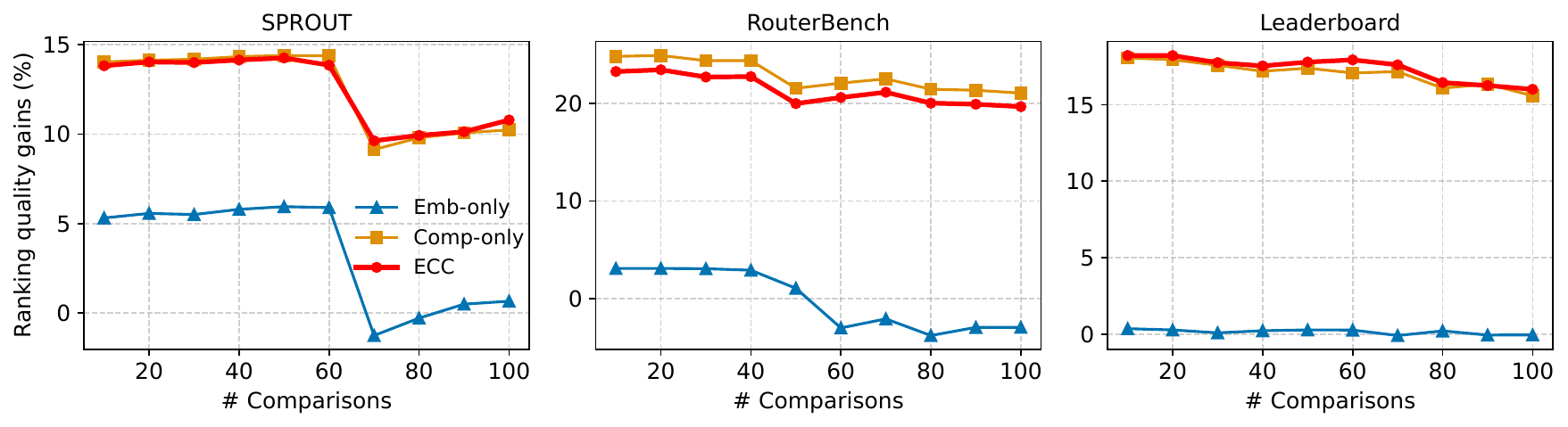}
    \caption{
    Performance of \tsf{ECC} on the downstream task of sample-efficient new model ranking across three benchmarks, where it consistently outperforms embedding-only clustering by achieving higher capability-aware ranking quality gains under the limited comparison budget.
    }
    \label{fig:ranking}
\end{figure}


\subsection{Ablation Studies}\label{app:abls}

\noindent\textbf{Number of Clustering Comparisons.}
To assess the sensitivity of \tsf{ECC}'s performance to the amount of comparison evidence used in clustering, we vary the number of model comparison pairs from 2 to 24 across three benchmarks.
Figure~\ref{fig:abl_pairs} shows that \tsf{ECC}'s ranking quality gains improve rapidly as the number of comparisons increases from very sparse to moderate levels, but then exhibit diminishing returns, and a slight decline at the highest comparison budgets. 
A similar trend is observed for comparison-only clustering.
These results suggest that a small number of comparisons is sufficient to capture most of the benefits of comparison signals, while additional comparisons may provide limited new information and accumulated noise, slightly degrading performance.

\noindent\textbf{Trade-off Parameter $\lambda$.}
We vary the trade-off parameter $\lambda$ from 0.5 to 50 to control the relative weight between embeddings and model comparison evidence during clustering, where a larger $\lambda$ places more weight on embeddings.  
Figure~\ref{fig:abl_lambda} shows a clear trend: performance improves rapidly as $\lambda$ increases from small values, peaks at an intermediate range, and then degrades as $\lambda$ becomes large.
This suggests that embeddings and model comparisons offer complementary information for clustering: over-emphasizing either signal hurts capability-aware ranking quality, while a balanced combination yields the best performance.

\noindent\textbf{Number of Queries Used for Clustering.}
We assess the sensitivity of different methods to the amount of clustering queries by varying the fraction of queries used for clustering from 0.2 to 0.9 across three benchmarks.
Figure~\ref{fig:abl_query} shows a clear trend: as more queries are used for clustering, all methods achieve higher ranking quality gains, with \tsf{ECC} consistently leading.
Even with only 20\% of queries for clustering, \tsf{ECC} achieves strong performance.
Overall, \tsf{ECC} is relatively robust to the number of clustering queries and delivers strong capability-aware rankings even under low-resource settings.

\noindent\textbf{Temperature $\T$.}
We vary $\T$ from 0.05 to 0.35 to assess the performance of different methods across three benchmarks.
Figure~\ref{fig:abl_temp} shows that clustering performance is strong with the low $\T$ regime and then degrades rapidly as $\T$ increases, consistent with over-smoothing that weakens posterior signals.
Notably, comparison-only clustering and \tsf{ECC} maintain high performance over a broader range of small $\T$ values, while embedding-only clustering deteriorates more steadily as $\T$ increases.
These results suggest that selecting an appropriate $\T$ is important for preserving informative responsibilities and achieving strong capability-aware ranking quality.

\noindent\textbf{Number of Probe Comparisons.}
To assess the sensitivity of \tsf{ECC} to the amount of posterior evidence available at inference time, we vary the number of probe comparisons per test query from 1 to 6. 
As shown in Figure~\ref{fig:probe}, increasing the number of probes consistently improves ranking quality gains for \tsf{ECC} and comparison-only clustering, while embedding-only clustering remains stable. 
This is expected: more probe comparisons provide stronger posterior evidence to infer the LLM capabilities for a given query, leading to more accurate query-specific capability-aware rankings. 
In contrast, embedding-only clustering relies solely on embedding signals to form clusters and thus is largely insensitive to the number of probe comparisons during inference.

\noindent\textbf{Probe Label Randomization.}
We further investigate the impact of probe comparisons at inference by corrupting probe comparison outcomes.
Specifically, for each probe comparison, we replace its binary label with a random draw from $\{0,1\}$ with probability $p\in[0,1]$, and keep the original outcome with probability $1-p$.
As shown in Figure~\ref{fig:probe_random}, increasing $p$ consistently degrades the performance of \tsf{ECC} and comparison-only clustering across all three benchmarks, indicating that accurate probe outcomes provide essential query-specific evidence for inferring capability-aware rankings.
In contrast, embedding-only clustering exhibits weaker sensitivity to $p$, consistent with its exclusive reliance on prior semantic information rather than posterior comparison signals.
Notably, \tsf{ECC} remains strong performance under mild corruption (e.g., $p\le 0.2$), demonstrating a degree of tolerance to occasional noise in probe outcomes.

\noindent\textbf{Probe Selection Robustness.}
We further evaluate the robustness of \tsf{ECC} to the choice of probe comparisons at inference time.
For each query, we randomly select different probe comparisons from the available comparison set and repeat this process seven times.
We report the mean and standard deviation of ranking quality gains across probe selections for each clustering method.
As shown in Table~\ref{tab:var_probe}, \tsf{ECC} consistently achieves the highest ranking quality gains across all three benchmarks, outperforming all baselines while exhibiting low variance across probe selections.
These results demonstrate that \tsf{ECC} has strong robustness to probe choice, indicating stable and reliable capability inference under varying probe selections.

\noindent\textbf{Embedding Models.}
We evaluate the sensitivity of \tsf{ECC} to the choice of embedding model by replacing the default encoder with \texttt{SFR-Embedding-2\_R} and \texttt{text-embedding-3-large}.
As shown in Figure~\ref{fig:emb} (SPROUT), \tsf{ECC} consistently outperforms the embedding-only clustering across varying numbers of clusters under both embedding choices.
These results suggest that \tsf{ECC} is robust to the embedding backbone and maintains strong performance across different encoders.

\noindent\textbf{Robustness to Tied Comparisons.}
To assess the robustness of \tsf{ECC} to tied comparisons, we extend the BT objective to allow soft pairwise targets, encoding wins, losses, and ties as $y=1$, $y=0$, and $y=0.5$, respectively.
Specifically, for each query, we vary the fraction of tied comparisons among all model comparisons from $5\%$ to $30\%$ and compare \tsf{ECC} with embedding-only and comparison-only clustering.
As shown in Figure~\ref{fig:tie}, \tsf{ECC} consistently achieves the best ranking quality gains across benchmarks and tie ratios.
These results show that \tsf{ECC} maintains its advantage over baselines across tied-comparison ratios, demonstrating its robustness to tie-aware comparison settings.

\subsection{Out-of-Distribution Generalization}\label{app:ood}
To evaluate \tsf{ECC} in an out-of-distribution (OOD) setting, we conduct a cross-dataset experiment using MMLU-Pro as the clustering set and MATH as the test-query set.
We compare \tsf{ECC} with embedding-only clustering and report ranking quality gains in~\Cref{tab:ood}.
The results show that \tsf{ECC} substantially outperforms the baseline under this distribution shift, demonstrating its robustness to OOD queries.

\subsection{Robustness to Different Judges}\label{app:diff_judges}
To assess robustness to different evaluation judges, we use MMLU-Pro as a shared task because it appears in both SPROUT and Open LLM Leaderboard v2, which use different judges to score model responses.
We then evaluate whether \tsf{ECC}'s gains remain consistent under these different judge settings.
As shown in~\Cref{tab:diff_judges}, \tsf{ECC} consistently outperforms embedding-only clustering across judges, demonstrating that its gains are robust to the choice of evaluator.

\subsection{Guided Optimal Query Routing}\label{app:qr}
We consider an important downstream application of capability-aware ranking: \emph{optimal query routing}, where the goal is to route each query to the model that yields the highest response quality. 
For each clustering method, we construct query clusters with the clustering queries.
During evaluation, for each test query, we randomly sample a single probe comparison (see \Cref{app:dis_cost} for a discussion of probe cost) and combine it with the query embedding to infer a query-specific capability-aware ranking.
The query is then routed to the top-ranked model under this inferred ranking, and we record its response score.
We report the average response quality across all test queries and compare against three baselines:
(i) Random routing, which routes each query to a randomly selected model,
(ii) Global routing, which routes every query to the top model according to the single global BT ranking, and
(iii) P2L~\citep{frick2025prompttoleaderboard}, which learns a black-box LLM predictor to directly output query-conditioned LLM leaderboard.
As shown in Figure~\ref{fig:routing}, \tsf{ECC} consistently achieves the highest average response quality across all three benchmarks and a wide range of $K$, outperforming all baselines.
Specifically, with $K=30$, \tsf{ECC} improves the response quality of routed queries by \textbf{16.6\%} on average over embedding-only clustering across three benchmarks.
This demonstrates that \tsf{ECC} learns a more capability-relevant clustering structure that enables better prediction of the best-performing model for a given query, even with only a single probe comparison, leading to improved routing performance.

\subsection{Sample-Efficient New Model Ranking}\label{app:er}
One important downstream application is to efficiently place a newly arrived model into an existing model ranking under a limited comparison budget.
We evaluate sample efficiency by measuring how accurately each clustering method ranks a held-out new model as a function of the number of queried comparisons (one new-model comparison per selected query).
Concretely, each clustering method first learns query clusters using only pairwise comparisons among the existing (old) models.
Given the learned clusters, we actively acquire comparisons involving the new model: for each selected query, we use its cluster responsibilities to select the median-ranked old model under the corresponding capability-aware ranking and query the pairwise comparison between the new model and this opponent.
After collecting a fixed budget of such comparisons, we keep the clusters learned from old-model comparisons fixed and update only the BT models by incorporating the acquired new-model comparisons.
We then evaluate on held-out new-model comparisons, and report the ranking quality gains relative to a single global BT model trained with the same old-model comparisons and the same new-model comparison budget but without clustering.

We vary the new-model comparison budget from 10 to 100 and treat each model in turn as the new model, while limiting the number of comparisons among the old models to at most 3 per query.
We report the median ranking quality gains over new-model runs on three benchmarks.
As shown in Figure~\ref{fig:ranking}, \tsf{ECC} achieves substantial ranking quality gains under small budgets, consistently outperforming embedding-only clustering, and is comparable to comparison-only clustering, demonstrating its superior effectiveness in leveraging limited comparison evidence for new-model ranking.
Notably, ranking quality gains remain relatively stable as the budget increases across all methods.
With a 100-comparison budget, \tsf{ECC} achieves \textbf{16.3} pp higher new-model ranking quality than embedding-only clustering on average.
This is expected: additional comparisons benefit both the global BT model and clustering methods, thus the relative advantage persists.

\section{Theoretical Proofs}\label{app:proof}
\begingroup
\renewcommand{\thetheorem}{3.1} 
\begin{restatable}[Convexity]{lemma}{Lemm}
\label{lem:convexity}
Consider updating $\bm{\theta}^{(k)}$ for cluster $k$ by minimizing the weighted BT loss $\ell({\bm{\theta}^{(k)}}) = \sum_{q \in \Q} r_{qk} \ell_{\text{comp}}(q; \bm{\theta}^{(k)})$.
Then $\ell({\bm{\theta}^{(k)}})$ is convex in $\bm{\theta}^{(k)}$.
\end{restatable}
\endgroup
\begin{proof}
    With the fixed observed comparison outcome $y$, according to~\Cref{eq:btloss}, we have
    \begin{equation*}
    \begin{aligned}
        \ell({\bm{\theta}^{(k)}}) & = \sum_{q \in \Q} r_{qk} \ell_{\text{comp}}(q; \bm{\theta}^{(k)}) \\
        & = \sum_{q \in \Q} \frac{r_{qk}}{|\N_q|} \sum_{(i, j, y) \in\N_q} -\log P(y\mid i,j;\bm{\theta}^{(k)}) \\
        & = \sum_{q \in \Q} \frac{r_{qk}}{|\N_q|} \sum_{(i, j, y) \in\N_q} -\left[y \log \sigma(\theta_j^{(k)} - \theta_i^{(k)}) + (1-y) \log \sigma(\theta_i^{(k)} - \theta_j^{(k)}) \right] 
    \end{aligned}
    \end{equation*}
    Let $t=\theta_j^{(k)}-\theta_i^{(k)}$ and define
    \begin{equation*}
        \psi_y(t) \triangleq y\log(1+e^{-t}) + (1-y)\log(1+e^{t}).
    \end{equation*}
    With $-\log\sigma(t)=\log(1+e^{-t})$ and $-\log\sigma(-t)=\log(1+e^{t})$, the negative log-likelihood can be written as
    \begin{equation*}
    \phi_{ij,y}(\bm{\theta}^{(k)}) = \psi_y\!\left(\theta_j^{(k)}-\theta_i^{(k)}\right)=\psi_y(t).
    \end{equation*}
    Both $\log(1+e^{t})$ and $\log(1+e^{-t})$ are \emph{log-sum-exp} functions and hence convex in $t$. Since $\psi_y(t)$ is a nonnegative weighted sum of these two convex functions for $y\in[0,1]$, $\psi_y$ is convex in $t$. 
    Furthermore, because $t=\theta_j^{(k)}-\theta_i^{(k)}$ is an affine function of $\bm{\theta}^{(k)}$, the composition rule implies that $\phi_{ij,y}(\bm{\theta}^{(k)})$ is convex in $\bm{\theta}^{(k)}$.

    Finally, $\ell(\bm{\theta}^{(k)})$ is a nonnegative weighted sum of convex terms, with weights $\frac{r_{qk}}{|\N_q|}\ge 0$, and is therefore convex in $\bm{\theta}^{(k)}$.
\end{proof}

\MainTheorem*
\begin{proof}
The \tsf{ECC} algorithm operates as a coordinate descent on the objective $\cL(\cP, \Theta)$, alternating between optimizing $\cP$ and $\Theta$.

\noindent\textbf{Optimize $\cP$.} 
In iteration $t$, we fix parameters $\bm \Theta^{(t)}$ and minimize~Objective~\ref{eq:objective} with respect to $\cP$ under the constraint $\sum_k r_{qk} = 1$. 
For each $q\in\Q$, the objective is:
\begin{equation}
    \min_{\bm r_q} \quad \sum_{k=1}^K r_{qk} \ell(q,k;\bm \Theta)^{(t)} - \T H(\bm{r}_q).
\end{equation}
Using $-\T H(\bm r_q)=\T\sum_{k=1}^K r_{qk}\log r_{qk}$, the subproblem becomes
\begin{equation}
\min_{\bm r_q} \sum_{k=1}^K r_{qk} \ell(q,k;\bm \Theta)^{(t)} + \T \sum_{k=1}^K r_{qk}\log r_{qk}.
\label{eq:subproblem}
\end{equation}
Since $\bm \Theta$ is fixed, the losses $\ell(q,k;\bm \Theta)^{(t)}$ are constants with respect to $\bm r_q$. 
Thus, the first term in Objective~\ref{eq:subproblem} is linear in $\bm r_q$, and the entropy term is convex for $\T>0$. 
Therefore, Objective~\ref{eq:subproblem} is convex in $\bm r_q$ and admits a unique minimizer. 
We form the Lagrangian with Lagrange multiplier $\alpha$ for the constraint $\sum_{k=1}^K r_{qk}=1$:

\begin{equation*}
\J(\bm r_q,\alpha) =\sum_{k=1}^K r_{qk}\ell(q,k;\bm \Theta)^{(t)} + \T\sum_{k=1}^K r_{qk} \log r_{qk} + \alpha\left(\sum_{k=1}^K r_{qk}-1\right).
\end{equation*}
Taking the derivative with respect to $r_{qk}$ and setting to 0 gives
\begin{equation*}
\frac{\partial\J}{\partial r_{qk}} = \ell(q,k;\bm \Theta)^{(t)} + \T(1+\log r_{qk}) + \alpha = 0 \Longrightarrow r_{qk}=\exp \left(-\frac{\ell(q,k;\bm \Theta)^{(t)}+\alpha}{\T}-1\right).
\end{equation*}
Normalizing with $\sum_k r_{qk}=1$ yields the closed-form Boltzmann solution:
\begin{equation*}
    r_{qk}^{(t+1)} = \frac{\exp(-\ell(q,k;\bm \Theta)^{(t)} / \T)}{\sum_{j=1}^K \exp(-\ell(q, j;\bm \Theta)^{(t)} / \T)}.
\end{equation*}
This is exactly the responsibility update used in~\Cref{alg:ecc}. 
Since this is the analytical solution to the minimization problem, we have $\cL(\cP^{(t+1)}, \bm \Theta^{(t)}) \leq \cL(\cP^{(t)}, \bm\Theta^{(t)})$.

\noindent\textbf{Optimize $\Theta$.}
We fix responsibilities $\cP^{(t+1)}$ and minimize~Objective~\ref{eq:objective} with respect to $\bm \Theta$. 
Since the entropy term is constant with respect to $\bm \Theta$, the problem decouples into independent optimizations for each cluster $k$:
\begin{itemize}
    \item \textbf{Update $\bm c_k$.}
    With $\cP^{(t+1)}$ and $\bm x_q$ fixed, we minimize with respect to $\bm c_k$:
    \begin{equation*}
    \min_{\|\bm c_k\|_2=1} -\sum_{q\in\Q} r_{qk}^{(t+1)} \bm x_q^\top \bm c_k.
    \end{equation*}
    Let $\bm v_k \triangleq \sum_{q\in\Q} r_{qk}^{(t+1)} \bm x_q$.
    The problem is equivalent to
    \begin{align}
    \max_{\|\bm c_k\|_2=1} \bm v_k^\top \bm c_k.
    \end{align}
    By the Cauchy-Schwarz inequality, $\bm v_k^\top \bm c_k \leq \|\bm v_k\|_2 \|\bm c_k\|_2 = \|\bm v_k\|_2$,
    and equality holds when $\bm c_k$ is aligned with $\bm v_k$.
    Therefore, with $\bm v_k \neq \bm 0$, the global minimizer has the closed form
    \begin{equation*}
    \bm c_k^{(t+1)} = \bm v_k / \|\bm v_k\|_2,
    \end{equation*}
    which is the normalized weighted mean of $\{\bm x_q\}$ under weights $\{r_{qk}^{(t+1)}\}$.
    
    \item \textbf{Update $\bm\theta^{(k)}$.}
    With $\cP^{(t+1)}$ fixed, we optimize only $\bm\theta^{(k)}$ for cluster $k$:
    \begin{equation*}
    \min_{\bm\theta^{(k)}}\; \sum_{q\in\Q} r_{qk}^{(t+1)} \ell_{\text{comp}}(q;\bm\theta^{(k)}).
    \end{equation*}
    By Lemma~\ref{lem:convexity}, this subproblem is convex in $\bm\theta^{(k)}$, thus any gradient-based optimization finds the global minimum.
\end{itemize}
Thus, we achieve $\cL(\cP^{(t+1)}, \bm\Theta^{(t+1)}) \leq \cL(\cP^{(t+1)}, \bm\Theta^{(t)})$.

Combining the two inequalities gives a monotone decrease:
\begin{equation}
    \cL(\cP^{(t+1)},\bm\Theta^{(t+1)}) \leq \cL(\cP^{(t)},\bm\Theta^{(t)}).
\end{equation}

\noindent\textbf{Boundedness.}
Since $0\le H(\bm{r}_q)\le \log K$, $\|\bm{x}_q\|_2=\|\bm{c}_k\|_2=1$ implies $\bm{x}_q^\top \bm{c}_k\leq 1$, and $\ell_{\text{comp}}(q;\bm{\theta}^{(k)})\geq 0$, we obtain the finite lower bound
\begin{equation}
    \begin{aligned}
        \cL(\cP,\bm\Theta) & \geq
        \sum_{q\in\Q} \sum_{k=1}^K r_{qk} \cdot \lambda\cdot (-1)
        -\T\sum_{q\in\Q} \log K \\
        & = -\lambda |\Q| - \T |\Q|\log K.
    \end{aligned}
\end{equation}
Hence $\{\cL^{(t)}\}$ is non-increasing and bounded below, so it converges to a finite limit.
\end{proof}

\section{Discussion}\label{app:discuss}
\subsection{Reducing Tied Comparisons }\label{app:diss_tie}
In practice, tied comparisons can arise when two models perform similarly or when the evaluator is uncertain.
Since ties provide weaker discriminative evidence than strict wins or losses and may reflect evaluator noise~\citep{zheng2023judging}, including them can dilute capability-aware signals.
Tie frequency can be reduced by using stronger evaluators or stricter judge prompts that reserve the tie option for cases where two responses are genuinely indistinguishable.

\subsection{Evaluation Judges}\label{app:dis_human}
Human judgments are an important source of preference signals in LLM evaluation~\citep{chiang2024chatbot}.
In practice, however, they are often collected from multiple annotators with heterogeneous preferences and varying reliability, which can introduce noise and systematic bias into pairwise comparisons.
Large-scale human evaluation is also substantially more expensive than automated evaluation.
Therefore, for each benchmark, we use its designated LLM-based judge or fixed evaluation protocol to obtain consistent pairwise signals within that benchmark and avoid annotator-dependent variability.

\subsection{Evaluation Efficiency of \tsf{ECC}}\label{app:dis_cost}
\tsf{ECC} is lightweight and does not require dense compute (e.g., GPU training~\citep{frick2025prompttoleaderboard}).
Its optimization alternates between (i) updating cluster-specific BT capability profiles and (ii) updating soft responsibilities, both of which are efficient CPU operations.
In our main setting with $K=30$, \tsf{ECC} converges within 32 iterations on SPROUT and RouterBench, and within 29 iterations on Open LLM Leaderboard.
Furthermore, fitting BT models from limited pairwise comparisons is substantially cheaper than training additional neural components with gradient-based optimization.
The main efficiency benefit of \tsf{ECC} lies in reducing evaluation cost.
Whereas some existing methods rely on exhaustively judging responses from all models for every query, \tsf{ECC} only requires a small subset of model outputs and a limited number of pairwise comparisons per query.
As shown in~\Cref{fig:abl_pairs}, using around three comparisons per query already yields strong clustering performance.
While \tsf{ECC} may require one probe signal for an unseen query (i.e., generating two model responses and performing one pairwise comparison), this overhead is $O(1)$ per query and does not scale with the number of candidate LLMs ($M$). 
Moreover, since the probe signal is robust to the choice of probing models (Table~\ref{tab:var_probe}), this probe can be obtained using lightweight models when evaluation cost is a concern.
Thus, \tsf{ECC} provides an evidence-efficient approach to LLM capability evaluation compared with methods that require substantially more model outputs or pairwise judgments.

\section{Extended Related Work}\label{app:related}
\noindent\textbf{Capability-Aware Clustering for LLM Evaluation.}
Recent work has moved beyond static, human-defined taxonomies and curated benchmark groupings~\citep{wang2019superglue, ying2024mmt, liu2024mmbench, srivastava2023beyond, al2024unibench} toward inducing latent capability structure from evaluation data for more fine-grained analysis~\citep{tian2025skillverse, moayeri2025unearthing, wei2025kcluster, luo2024llm, murahari2024qualeval}.
\citet{tian2025skillverse} introduces {SkillVerse}, a diagnosis framework that utilizes an ``LLM-as-a-Judge'' to critique responses and organize them into a hierarchical skill dendrogram, revealing granular capabilities often missed by human labeling. 
\citet{moayeri2025unearthing} extracts latent task attributes, such as subtasks and domains, from model evaluation data, and uses them as signals to inform unsupervised clustering of prompts.
On a related vein, \citet{wei2025kcluster} proposes {KCluster}, which leverages an LLM-induced question congruity metric to cluster assessment items into latent {knowledge components}, yielding clusters that predict student and model performance more accurately than expert-defined categories.
However, these methods typically induce capability structure through \emph{explicit} skill or attribute representations for diagnosis or performance prediction. 
In contrast, \tsf{ECC} does not rely on explicit skill extraction; instead, it directly integrates limited posterior evidence from pairwise model comparisons with prior semantic embeddings to learn \emph{soft} capability-aligned clusters tailored to LLM ranking over diverse queries. 
Furthermore, whereas prior work often yields discrete partitions that support group-level analysis once the grouping is fixed, \tsf{ECC} learns soft capability-aligned clusters that support \emph{query-conditional} inference over LLM capabilities, enabling stronger adaptation and generalization under limited comparison evidence.

\noindent\textbf{Sample-Efficient LLM Evaluation.}
Modern LLM evaluation is costly as a comprehensive assessment requires scoring all LLMs over thousands of benchmark records~\citep{li2025adaptivetestingllmevaluation, chiang2024chatbot}.
To reduce evaluation cost, recent work has developed sample-efficient methods based on adaptive item selection, benchmark compression, and comparison-based ranking~\citep{li2025adaptivetestingllmevaluation, kipnis2025metabench, chiang2024chatbot, son2025arena, song2025irt}.
For instance, Chatbot Arena~\citep{chiang2024chatbot} collects crowdsourced pairwise preferences and fits Bradley-Terry models~\citep{bradley1952rank} to produce scalable global leaderboards without exhaustively evaluating every model on every prompt.
In parallel, Item Response Theory (IRT) provides a principled measurement framework for improving sample efficiency~\citep{baker2001basics}.
\citet{song2025irt} proposed {IRT-Router}, which models latent model ability and query difficulty to support cost-effective multi-LLM routing, improving over semantic-routing baselines.
{ATLAS}~\citep{li2025adaptivetestingllmevaluation} uses IRT-based adaptive testing (e.g., Fisher-information-guided item selection) to estimate model ability accurately with far fewer benchmark questions, while {Metabench}~\citep{kipnis2025metabench} leverages redundancy in large-scale leaderboard data to distill sparse benchmarks that preserve score estimates.
Furthermore, {Arena-Lite}~\citep{son2025arena} improves ranking reliability under limited budgets by organizing LLM-judge comparisons in a tournament structure, reducing the number of required comparisons.
However, most prior work primarily targets sample-efficient estimation of \emph{global} model scores or leaderboards, which is a single ranking shared across all queries.
In contrast, we target \emph{capability-aware} ranking under heterogeneous query demands: \tsf{ECC} learns latent capability clusters and performs per-query posterior inference to produce a \emph{query-specific} mixture ranking over LLMs.
This enables fine-grained adaptation across queries while remaining effective under limited comparison budgets.

\noindent\textbf{Query Routing}.
Recent work studies how to match heterogeneous queries to the most suitable model under quality-cost constraints.
A common line of approaches invokes multiple models for each query, either by ensembling or reranking candidate responses~\citep{jiang2023llm, wang2023fusing} or by cascading models from cheaper to more expensive ones until a stopping criterion is met~\citep{chen2023frugalgpt, aggarwal2024automix, yue2023large, lee2023orchestrallm, song2025irt}. 
For instance, LLM-Blender~\citep{jiang2023llm}  combines and reranks responses from multiple LLMs to produce a final output, and Frugal-GPT~\citep{chen2023frugalgpt} sequentially queries models in increasing capability to reduce cost while maintaining quality.
Another line trains a router on historical or auxiliary data to estimate the performance and cost of each query and select the optimal LLM accordingly~\citep{ong2025routellmlearningroutellms, somerstep2025carrot, hu2024routerbench, ding2024hybrid, lu2023routing, wu2025efficient, shnitzer2023large, feng2024graphrouter, hari2023tryage, wang2025mixllm, stripelis2024tensoropera, jin2025radialrouter}.
For example, RouterLLM~\citep{ong2025routellmlearningroutellms} trains a lightweight BERT model on preference-labeled data to improve selection accuracy, and CARROT~\citep{somerstep2025carrot} learns predictors of both quality and cost to support flexible trade-offs.
RadialRouter~\citep{jin2025radialrouter} proposes a lightweight Transformer backbone, RadialFormer, to model query-LLM relationships for routing, while GraphRouter~\citep{feng2024graphrouter} uses a heterogeneous graph to capture interactions between query requirements and model capabilities.
These routing methods typically learn a per-query selector, whereas \tsf{ECC} learns a capability-aware clustering structure and outputs a query-specific mixture ranking over models, with routing as a special case (selecting the top-ranked model) while enabling broader capability-aware evaluation tasks.
One recent work~\citep{frick2025prompttoleaderboard} introduces Prompt-to-Leaderboard (P2L), which trains an LLM-based predictor on human-annotated pairwise comparisons to directly estimate query-conditioned BT coefficients, thereby inducing a per-query LLM leaderboard.
In contrast, \tsf{ECC} differs fundamentally in both goal and methodology. P2L focuses on a specific application that learns an amortized black-box predictor for per-query leaderboard generation (e.g., for query routing), without modeling any interpretable latent capability structure across queries. 
\tsf{ECC}, on the other hand, aims to uncover a fundamental capability-aware clustering structure that reveals how queries relate based on underlying capability demands -- a structure that is both interpretable and broadly applicable to multiple downstream tasks, including query routing and new model ranking.
Furthermore, P2L requires costly model training, whereas \tsf{ECC} is training-free and efficient, and executable on commodity hardware such as CPUs.

\section{Broader Impact}\label{app:imp}
This paper advances machine learning by proposing a clustering framework that enables more fine-grained, capability-aware LLM ranking across queries with diverse capability requirements. 
As a methodological contribution, it introduces limited direct societal risk. 
By providing more reliable LLM rankings, it can help users choose models that better match their needs, improving downstream application outcomes for a broader range of users.

\section{LLM Summaries for Cluster Overlap Qualitative Study}\label{app:cluster}
\subsection{Summarization Template}\label{app:prompts}

\begin{tcolorbox}[left=1mm, right=1mm, top=1mm, bottom=1mm, boxrule=0.9pt, breakable]
\begin{Verbatim}[breaklines=true]
You are given two clusters of prompts, Cluster_A and Cluster_B.
They are provided as two JSON arrays (lists of prompt strings):
the first array is Cluster_A, and the second array is Cluster_B.

Fairness constraints:
- Treat A and B symmetrically; do not assume any meaning from IDs or file names.
- Use ONLY the provided prompts. Do NOT use external dataset knowledge.
- First scan ALL prompts in each cluster to do a comprehensive analysis, then select representative evidence for the response.
- Do NOT judge which cluster is better.

Definitions:
- Semantic topic (surface): the dominant surface theme, stated in a way that distinguishes clusters
- Hidden capability: an underlying skill/behavior pattern/dominant bottleneck beyond surface semantics (e.g., multi-step constraint solving, symbolic manipulation, representation parsing, format compliance).

Task (MUST follow this order):

1) Cluster A (independent analysis)
- Topic scope: 1-2 sentences describing the surface semantic region (topic family).
- Hidden capabilities: 1-3 items (choose the smallest number that faithfully explains most prompts).
  For each capability:
  - name: concise, accurate
  - evidence_snippets: quote exactly 2 short snippets (<=12 words each),
    taken verbatim from DIFFERENT prompts in A, supporting this capability.
- Key patterns: 0-5 bullet points (task type, reasoning style, constraints, formatting, etc.).

2) Cluster B (independent analysis)
(Same required structure as A.)

3) Comparison (A vs B)
- Shared semantic topic region: 1 sentence (no boolean).
- Shared hidden-capability region: 1 sentence (no boolean), describing the broad capability family in 3-8 words.
- Shared hidden capabilities: list capability names that both clusters appear to require.
- Systematic Differences:
  - Capability differences: 0-4 bullets.
  - Semantic topic differences: 0-4 bullets.
  Each bullet should reference at least one evidence snippet already quoted above.

Return ONLY a JSON object:

{
  "cluster_A_summary": {
    "topic": "...",
    "hidden_capabilities": [
      { "name": "...", "evidence_snippets": ["...", "..."] }
    ],
    "key_patterns": ["..."]
  },
  "cluster_B_summary": {
    "topic": "...",
    "hidden_capabilities": [
      { "name": "...", "evidence_snippets": ["...", "..."] }
    ],
    "key_patterns": ["..."]
  },
  "comparison": {
    "shared_semantic_topic_region": "...",
    "shared_hidden_capability_region": "...",
    "shared_hidden_capabilities": ["..."],
    "capability_differences": ["..."],
    "semantic_topic_differences": ["..."],
    "evidence_summary": {
      "semantic_similarity_signal": "1-2 sentences describing BOTH (i) strongest semantic overlap and (ii) strongest semantic divergence, citing snippets already quoted above.",
      "capability_similarity_signal": "1-2 sentences describing BOTH (i) strongest capability overlap and (ii) strongest capability divergence, citing snippets already quoted above."
    }
  }
}
\end{Verbatim}
\end{tcolorbox} 

\subsection{Cell Comparison:(\texttt{ECC\_1}$\cap$\texttt{Emb\_1}) vs.\ (\texttt{ECC\_10}$\cap$\texttt{Emb\_1})}
\begin{tcolorbox}[left=1mm, right=1mm, top=1mm, bottom=1mm, boxrule=0.9pt, breakable]
\begin{Verbatim}[breaklines=true]
{
  "cluster_A_summary": {
    "topic": "Chemistry/biochemistry prompts centered on explaining and predicting reaction behavior (mechanisms, rate/selectivity) and linking molecular/coordination structure to function, often invoking spectroscopy or computational chemistry as the evidence source. ",
    "hidden_capabilities": [
      {
        "name": "Mechanism-based causal reasoning",
        "evidence_snippets": [
          "mechanism of the reaction between benzene and bromine",
          "reaction mechanism influence the rate of the reaction"
        ]
      },
      {
        "name": "Method-driven structure/property inference",
        "evidence_snippets": [
          "using infrared spectroscopy and proton NMR spectroscopy data interpretation",
          "based on Density Functional Theory calculations"
        ]
      }
    ],
    "key_patterns": [
      "Frequent 'how does X affect Y' framing (conditions → rate/selectivity/structure).",
      "Requires mapping structural/electronic features to reactivity or function (often metals/coordination).",
      "Often asks to justify claims with a method or evidence channel (spectroscopy, DFT, MD, docking).",
      "Many prompts demand step-by-step mechanisms, intermediates, or catalytic roles."
    ]
  },
  "cluster_B_summary": {
    "topic": "Applied chemistry prompts spanning drug/antibiotic design, catalysis, and biochemical pathways, with many questions about optimizing efficacy/safety and using simulations/analytical methods to justify decisions. ",
    "hidden_capabilities": [
      {
        "name": "Mechanism-based causal reasoning",
        "evidence_snippets": [
          "mechanism for the electrophilic addition of hydrogen bromide",
          "mechanism of ligand exchange in metal-organic frameworks (MOFs)"
        ]
      },
      {
        "name": "Constraint-aware molecular design",
        "evidence_snippets": [
          "specifically targets the cell wall of Gram-negative bacteria",
          "without harming healthy cells in the body"
        ]
      },
      {
        "name": "Simulation/measurement protocol design",
        "evidence_snippets": [
          "Use Molecular Dynamics simulations to investigate the interactions",
          "Using quantum chemistry calculations, analyze the effects of solvent polarity"
        ]
      }
    ],
    "key_patterns": [
      "Design prompts emphasize specificity, off-target minimization, and safety/efficacy tradeoffs.",
      "Mixes mechanistic organic chemistry (addition/substitution) with bio/medical contexts.",
      "Repeated parameter-sweep reasoning (temperature/pH/concentration) and optimization framing.",
      "Often requests an explicit method plan (MD/QC/analytical technique) rather than only explanation.",
      "Occasional strict output schema (e.g., 'The answer is ...')"
    ]
  },
  "comparison": {
    "shared_semantic_topic_region": "Both clusters sit in advanced chemistry/biochemistry questions about mechanisms, structure–property links, and method-supported analysis. ",
    "shared_hidden_capability_region": "Mechanistic, method-backed scientific reasoning",
    "shared_hidden_capabilities": [
      "Mechanism-based causal reasoning",
      "Method-driven structure/property inference"
    ],
    "capability_differences": [
      "B more often adds explicit design constraints (\"without harming healthy cells in the body\") compared with A’s analysis-first framing (\"based on Density Functional Theory calculations\").",
      "A more consistently foregrounds inferring structure/properties from analytic evidence (\"using infrared spectroscopy and proton NMR spectroscopy data interpretation\") than B’s broader protocol-planning (\"Use Molecular Dynamics simulations to investigate the interactions\").",
      "B includes more end-to-end synthesis/therapy-oriented optimization (\"specifically targets the cell wall of Gram-negative bacteria\") than A’s reaction-mechanism focus (\"mechanism of the reaction between benzene and bromine\")."
    ],
    "semantic_topic_differences": [
      "A leans toward mechanistic/physical and coordination-chemistry explanations (\"reaction mechanism influence the rate of the reaction\").",
      "B leans toward biomedical/drug-design targets and applied constraints (\"specifically targets the cell wall of Gram-negative bacteria\")."
    ],
    "evidence_summary": {
      "semantic_similarity_signal": "Strongest overlap is mechanistic chemistry phrasing (\"reaction mechanism influence the rate of the reaction\"; \"mechanism for the electrophilic addition of hydrogen bromide\"). Strongest divergence is A’s spectroscopy/DFT emphasis (\"using infrared spectroscopy and proton NMR spectroscopy data interpretation\") versus B’s therapeutic targeting language (\"specifically targets the cell wall of Gram-negative bacteria\").",
      "capability_similarity_signal": "Both require translating a prompt into a causal account with justified steps (\"mechanism of the reaction between benzene and bromine\"; \"mechanism of ligand exchange in metal-organic frameworks (MOFs)\"). Divergence is that B more frequently layers constraint satisfaction (\"without harming healthy cells in the body\") on top of the scientific reasoning, while A more often centers evidence-based inference (\"based on Density Functional Theory calculations\")."
    }
  }
}
\end{Verbatim}
\end{tcolorbox}

\subsection{Cell Comparison: (\texttt{ECC\_1}$\cap$\texttt{Emb\_10}) vs.\ (\texttt{ECC\_10}$\cap$\texttt{Emb\_10})}
\begin{tcolorbox}[left=1mm, right=1mm, top=1mm, bottom=1mm, boxrule=0.9pt, breakable]
\begin{Verbatim}[breaklines=true]
{
  "cluster_A_summary": {
    "topic": "Materials chemistry prompts centered on tuning nano/molecular material properties (electronic/optical/photochemical, catalytic, interfacial) via composition, size/shape, and surface/processing choices, often framed around semiconductors, nanoparticles, polymers, and functional materials. ",
    "hidden_capabilities": [
      {
        "name": "Variable→property sensitivity reasoning",
        "evidence_snippets": [
          "dopant concentration affect these properties",
          "variations in their size, shape, and composition"
        ]
      },
      {
        "name": "Physics-based computational prediction",
        "evidence_snippets": [
          "density functional theory (DFT)",
          "Using molecular dynamics simulations, investigate"
        ]
      },
      {
        "name": "Synthesis route optimization planning",
        "evidence_snippets": [
          "synthesizing carbon nanotubes with controlled chirality",
          "optimized to improve their catalytic performance"
        ]
      }
    ],
    "key_patterns": [
      "Frequent causal framing: how a controllable factor changes a property.",
      "Strong emphasis on nanoscale structure/property links (size/shape/composition).",
      "Recurring computational-method prompts (DFT/MD) asking to predict/compare properties.",
      "Often posed as “optimize/design a method/material” toward performance targets."
    ]
  },
  "cluster_B_summary": {
    "topic": "Materials/polymer engineering prompts emphasizing practical optimization of synthesis/processing and characterization for application outcomes (e.g., 3D printing, sensors, drug delivery, coatings, photovoltaics, colloids/emulsions). ",
    "hidden_capabilities": [
      {
        "name": "Process parameter optimization reasoning",
        "evidence_snippets": [
          "optimized to achieve the highest absorption capacity",
          "processing parameters or equipment can be made"
        ]
      },
      {
        "name": "Application-driven system design",
        "evidence_snippets": [
          "develop a novel drug delivery system",
          "polymer-based sensor is suitable for detecting heavy metal ions"
        ]
      },
      {
        "name": "Experimental protocol articulation",
        "evidence_snippets": [
          "step-by-step explanation of the DMA technique",
          "determine the molecular weight distribution"
        ]
      }
    ],
    "key_patterns": [
      "More explicit “how-to” framing around optimization of processing/synthesis conditions.",
      "Application contexts recur (biomedical delivery, sensing, 3D printing, coatings).",
      "More prompts request measurement/characterization workflows and technique choice.",
      "Often asks for tradeoffs among performance, stability, and practical constraints."
    ]
  },
  "comparison": {
    "shared_semantic_topic_region": "Both clusters sit in materials/polymer/nanomaterials chemistry and engineering, focusing on how controllable design or processing choices impact functional properties.",
    "shared_hidden_capability_region": "property–performance optimization reasoning",
    "shared_hidden_capabilities": [
      "Variable→property sensitivity reasoning",
      "Synthesis route optimization planning",
      "Process parameter optimization reasoning"
    ],
    "capability_differences": [
      "A leans more on first-principles/computational inference (\"density functional theory (DFT)\") while B leans more on hands-on characterization workflows (\"step-by-step explanation of the DMA technique\").",
      "B foregrounds operational process tuning (\"processing parameters or equipment can be made\") more than A’s emphasis on mechanistic variable sweeps at the material level (\"variations in their size, shape, and composition\").",
      "A’s optimization is often about creating/tuning advanced materials (\"synthesizing carbon nanotubes with controlled chirality\"), whereas B’s optimization is often framed as improving deployable performance in applied systems (\"develop a novel drug delivery system\")."
    ],
    "semantic_topic_differences": [
      "A skews toward nano/optoelectronic and structure-driven property questions (\"variations in their size, shape, and composition\"), while B more often anchors to manufacturing/processing contexts (\"processing parameters or equipment can be made\").",
      "B includes more explicit biomedical/sensing application prompts (\"develop a novel drug delivery system\"; \"polymer-based sensor is suitable for detecting heavy metal ions\"), which are less foregrounded in the A evidence."
    ],
    "evidence_summary": {
      "semantic_similarity_signal": "Strong overlap: both ask how controllable knobs change functional properties (e.g., \"dopant concentration affect these properties\" and \"optimized to achieve the highest absorption capacity\"). Strong divergence: A emphasizes nanoscale structure/composition levers (\"variations in their size, shape, and composition\"), while B emphasizes processing/implementation levers (\"processing parameters or equipment can be made\").",
      "capability_similarity_signal": "Strong overlap: both demand optimization-style reasoning over materials decisions (\"optimized to improve their catalytic performance\"; \"optimized to achieve the highest absorption capacity\"). Strong divergence: A more often requires physics-based modeling/prediction (\"density functional theory (DFT)\") whereas B more often requires procedural measurement articulation (\"step-by-step explanation of the DMA technique\")."
    }
  }
}
\end{Verbatim}
\end{tcolorbox}

\subsection{Cell Comparison: (\texttt{ECC\_1}$\cap$\texttt{Emb\_1}) vs.\ (\texttt{ECC\_1}$\cap$\texttt{Emb\_10})}
\begin{tcolorbox}[left=1mm, right=1mm, top=1mm, bottom=1mm, boxrule=0.9pt, breakable]
\begin{Verbatim}[breaklines=true]
{
  "cluster_A_summary": {
    "topic": "Chemistry/biochemistry prompts centered on reaction behavior (rate/selectivity/mechanism) and structure–function questions for molecules, enzymes, and metal complexes, often framed as “how does X affect Y?” ",
    "hidden_capabilities": [
      {
        "name": "Parameter–response reasoning",
        "evidence_snippets": [
          "concentration of the reactants affect the rate of the chemical reaction",
          "reaction mechanism for the conversion of an alcohol to an alkene"
        ]
      },
      {
        "name": "Simulation-based property prediction",
        "evidence_snippets": [
          "molecular docking studies be used to identify potential inhibitors",
          "calculated using quantum chemistry methods"
        ]
      }
    ],
    "key_patterns": [
      "Frequent causal prompts varying conditions (concentration, pH, light intensity, catalyst).",
      "Emphasis on explaining mechanisms/intermediates and linking them to rate/selectivity.",
      "Regular calls for computational approaches (quantum chemistry, docking, molecular dynamics).",
      "Often requests examples, comparisons, or experimental/measurement evidence."
    ]
  },
  "cluster_B_summary": {
    "topic": "Materials/nanoscience prompts about tuning and optimizing material properties (optical/electronic/mechanical) for devices and applications (e.g., displays, solar cells), commonly via size/shape/composition/surface changes. ",
    "hidden_capabilities": [
      {
        "name": "Parameter–response reasoning",
        "evidence_snippets": [
          "optical and electrical properties of liquid crystals be tuned",
          "size of a quantum dot affect its electronic and optical properties"
        ]
      },
      {
        "name": "Synthesis/processing optimization",
        "evidence_snippets": [
          "synthesis process of carbon nanotubes be optimized for large-scale production",
          "surface modification of nanoparticles affect their stability in a given solvent"
        ]
      },
      {
        "name": "Simulation-based property prediction",
        "evidence_snippets": [
          "density functional theory calculations",
          "Monte Carlo simulations"
        ]
      }
    ],
    "key_patterns": [
      "Strong focus on material property optimization for performance (efficiency, durability, stability).",
      "Repeated size/shape/composition/surface-chemistry perturbations as levers for properties.",
      "Frequent ask for synthesis routes and scalable, cost-effective, or green processes.",
      "Common invocation of computational tools (DFT, Monte Carlo, quantum chemistry)."
    ]
  },
  "comparison": {
    "shared_semantic_topic_region": "Both clusters ask how changing conditions or structure alters physical/chemical behavior and performance-relevant properties.",
    "shared_hidden_capability_region": "Parameter–response causal modeling",
    "shared_hidden_capabilities": [
      "Parameter–response reasoning",
      "Simulation-based property prediction"
    ],
    "capability_differences": [
      "Cluster A more often centers mechanistic explanation of transformations and selectivity (e.g., \"reaction mechanism for the conversion of an alcohol to an alkene\"), whereas Cluster B more often centers end-to-end optimization of making/processing materials (e.g., \"synthesis process of carbon nanotubes be optimized for large-scale production\").",
      "Both reference computation, but Cluster A’s computation is frequently tied to molecular interaction/design (e.g., \"molecular docking studies be used to identify potential inhibitors\"), while Cluster B more frequently highlights materials-physics simulation toolkits (e.g., \"density functional theory calculations\")."
    ],
    "semantic_topic_differences": [
      "Cluster A foregrounds reaction/biomolecular settings (e.g., \"concentration of the reactants affect the rate of the chemical reaction\"), while Cluster B foregrounds engineered materials and device contexts (e.g., \"optical and electrical properties of liquid crystals be tuned\").",
      "Cluster B’s size-driven materials framing (e.g., \"size of a quantum dot affect its electronic and optical properties\") contrasts with Cluster A’s transformation-centric framing (e.g., \"reaction mechanism for the conversion of an alcohol to an alkene\")."
    ],
    "evidence_summary": {
      "semantic_similarity_signal": "Strongest overlap is the shared “vary X, observe Y” framing (\"concentration of the reactants affect the rate of the chemical reaction\" vs \"size of a quantum dot affect its electronic and optical properties\"). Strongest divergence is Cluster A’s reaction-transformation emphasis (\"reaction mechanism for the conversion of an alcohol to an alkene\") versus Cluster B’s device/material tuning emphasis (\"optical and electrical properties of liquid crystals be tuned\").",
      "capability_similarity_signal": "Strongest capability overlap is prediction via modeling/simulation (\"calculated using quantum chemistry methods\" alongside \"density functional theory calculations\" / \"Monte Carlo simulations\"). Strongest divergence is Cluster B’s emphasis on synthesis and surface/process levers (\"surface modification of nanoparticles affect their stability in a given solvent\"), whereas Cluster A more often frames computation around molecular binding/design questions (\"molecular docking studies be used to identify potential inhibitors\")."
    }
  }
}
\end{Verbatim}
\end{tcolorbox} 

\subsection{Cell Comparison: (\texttt{ECC\_10}$\cap$\texttt{Emb\_1}) vs.\ (\texttt{ECC\_10}$\cap$\texttt{Emb\_10})}
\begin{tcolorbox}[left=1mm, right=1mm, top=1mm, bottom=1mm, boxrule=0.9pt, breakable]
\begin{Verbatim}[breaklines=true]
{
  "cluster_A_summary": {
    "topic": "Chemistry and biochemistry prompts centered on designing therapeutics and explaining/quantifying chemical behavior (drug selectivity, reaction rates, and analytical measurements) across biomedical and environmental contexts. ",
    "hidden_capabilities": [
      {
        "name": "Selective intervention design",
        "evidence_snippets": [
          "selectively targets cancer cells while sparing healthy cells",
          "without affecting the related enzyme Cyclooxygenase-1 (COX-1)"
        ]
      },
      {
        "name": "Causal kinetics under parameter changes",
        "evidence_snippets": [
          "varying the concentration of a reactant",
          "using a rate equation for the reaction in question"
        ]
      },
      {
        "name": "Analytical quantification planning",
        "evidence_snippets": [
          "determine the concentration of caffeine in a sample of coffee",
          "concentration of lead in a water sample"
        ]
      }
    ],
    "key_patterns": [
      "Design/modify drugs and inhibitors with selectivity and side-effect constraints.",
      "Reason about how changing conditions (concentration, temperature, catalysts) alters outcomes.",
      "Frequent emphasis on measurement workflows (e.g., chromatography/spectroscopy) and interpreting what they reveal."
    ]
  },
  "cluster_B_summary": {
    "topic": "Materials-science prompts emphasizing polymers, nanoparticles, and functional materials, focusing on how composition/structure and processing choices govern mechanical, optical, electronic, and transport properties (often for energy, sensing, and coatings). ",
    "hidden_capabilities": [
      {
        "name": "Structure–property mapping",
        "evidence_snippets": [
          "size and shape of a gold nanoparticle",
          "degree of crystallinity of a polymer affect its mechanical properties"
        ]
      },
      {
        "name": "Multi-objective process optimization",
        "evidence_snippets": [
          "optimized while maintaining high selectivity",
          "optimize the printing conditions for a specific polymer-based material"
        ]
      },
      {
        "name": "Modeling/simulation method integration",
        "evidence_snippets": [
          "Use molecular dynamics simulations to investigate",
          "Monte Carlo simulation"
        ]
      }
    ],
    "key_patterns": [
      "Repeated structure/composition → performance questions (optical, mechanical, conductive, catalytic).",
      "Optimization framed as tuning synthesis/processing knobs (printing, membranes, blends, coatings).",
      "Regular calls for computational approaches (MD/MC/DFT) alongside characterization/measurement."
    ]
  },
  "comparison": {
    "shared_semantic_topic_region": "Both clusters live in chemistry-driven design-and-analysis problems that connect controllable factors to measurable outcomes. ",
    "shared_hidden_capability_region": "Structure–function optimization and causal analysis",
    "shared_hidden_capabilities": [
      "Optimization under constraints",
      "Causal reasoning over parameter changes",
      "Method selection and experimental/computational planning"
    ],
    "capability_differences": [
      "A foregrounds biological/therapeutic selectivity (\"selectively targets cancer cells while sparing healthy cells\"), while B foregrounds material structure→property links (\"degree of crystallinity of a polymer affect its mechanical properties\").",
      "B more often treats optimization as engineering/process tuning (\"optimize the printing conditions for a specific polymer-based material\"), whereas A more often formalizes changes via kinetic law/derivation (\"using a rate equation for the reaction in question\").",
      "B more explicitly integrates computational simulation workflows (\"Use molecular dynamics simulations to investigate\"; \"Monte Carlo simulation\"), while A’s exemplars here emphasize measurement/assay planning (\"determine the concentration of caffeine in a sample of coffee\"; \"concentration of lead in a water sample\")."
    ],
    "semantic_topic_differences": [
      "A’s surface topics skew toward drugs/biomedical chemistry (\"selectively targets cancer cells while sparing healthy cells\") and chemical analytics (\"concentration of lead in a water sample\").",
      "B’s surface topics skew toward polymers/nanomaterials and manufacturing constraints (\"size and shape of a gold nanoparticle\"; \"optimize the printing conditions for a specific polymer-based material\").",
      "B includes explicit selectivity–throughput tradeoffs in separation materials (\"optimized while maintaining high selectivity\"), whereas A’s tradeoffs are framed as efficacy vs off-target/side effects (\"without affecting the related enzyme Cyclooxygenase-1 (COX-1)\")."
    ],
    "evidence_summary": {
      "semantic_similarity_signal": "Strongest overlap is the shared focus on tuning controllable factors to reach a target outcome (A: \"varying the concentration of a reactant\"; B: \"optimized while maintaining high selectivity\"). Strongest divergence is the application domain framing—therapeutics/biological targeting in A (\"selectively targets cancer cells while sparing healthy cells\") versus materials processing/structure in B (\"size and shape of a gold nanoparticle\").",
      "capability_similarity_signal": "Both clusters require multi-factor causal reasoning and design under constraints (A: \"without affecting the related enzyme Cyclooxygenase-1 (COX-1)\"; B: \"optimized while maintaining high selectivity\"). The sharpest capability divergence is B’s repeated reliance on explicit simulation toolchains (\"Use molecular dynamics simulations to investigate\", \"Monte Carlo simulation\") versus A’s emphasis on formal kinetic/measurement formulations (\"using a rate equation for the reaction in question\", \"determine the concentration of caffeine in a sample of coffee\")."
    }
  }
}
\end{Verbatim}
\end{tcolorbox} 

\section{Example Prompts Referenced in Cluster Overlap Analysis}\label{app:examples}
\subsection{Example Prompts in Cell \texttt{ECC\_1}$\cap$\texttt{Emb\_1}}
\begin{tcolorbox}[left=1mm, right=1mm, top=1mm, bottom=1mm, boxrule=0.9pt, breakable]
\begin{Verbatim}[breaklines=true]
"What is the mechanism of the reaction between benzene and bromine to form bromobenzene?",
"What is the mechanism of the reaction between benzene and bromine in the presence of an iron catalyst? Describe the intermediates formed during the process and explain the role of the catalyst in the reaction.",
"Identify the unknown organic compound (molecular formula C10H12O) using infrared spectroscopy and proton NMR spectroscopy data interpretation.",
"How does the surface energy of a metal change when a foreign atom is adsorbed onto the surface, based on Density Functional Theory calculations?",
"How does the concentration of the reactants affect the rate of the chemical reaction and the overall yield of the product in the specific reaction between potassium iodide and hydrogen peroxide?",
"What is the reaction mechanism for the conversion of an alcohol to an alkene using a strong acid catalyst, and what factors affect the selectivity of this reaction?",
"What is the relationship between the electronic excited states and the optical properties of a certain molecule, and how can this be calculated using quantum chemistry methods?",
"How does temperature affect the phase transition of water from solid (ice) to liquid using Monte Carlo simulations in computational chemistry?"
\end{Verbatim}
\end{tcolorbox}

\subsection{Example Prompts in Cell \texttt{ECC\_10}$\cap$\texttt{Emb\_1}}
\begin{tcolorbox}[left=1mm, right=1mm, top=1mm, bottom=1mm, boxrule=0.9pt, breakable]
\begin{Verbatim}[breaklines=true]
"What is the mechanism of ligand exchange in metal-organic frameworks (MOFs) and how does it impact the properties and applications of the material? Compare and contrast the ligand exchange mechanisms observed in MOFs with those observed in coordination polymers.",
"How can we design a drug that specifically targets the cell wall of gram-negative bacteria, such as Escherichia coli, without affecting the cells of the human host?",
"How can we use structure-based drug design to develop a new antibiotic that specifically targets the cell wall of Gram-negative bacteria, without affecting the healthy cells of the host organism?",
"How can we design a drug that specifically targets and inhibits the growth of HER2-positive breast cancer cells without harming healthy cells in the body?",
"How does the surface chemistry of gold nanoparticles affect their stability and reactivity in a solvent environment? Use Molecular Dynamics simulations to investigate the interactions between the gold nanoparticle and the solvent molecules, and explore the role of various types of ligands on the stability of the nanoparticle.",
"How can we design a drug that selectively targets cancer cells while sparing healthy cells in the body?",
"How can you design a drug that specifically targets the enzyme Cyclooxygenase-2 (COX-2) without affecting the related enzyme Cyclooxygenase-1 (COX-1)?",
"How does varying the initial reactant concentration and temperature affect the rate of reaction with respect to reactant orientation? Express your answer quantitatively using a rate equation for the reaction in question.",
"How does the binding affinity of a protein-ligand complex change with increasing temperature? Use molecular dynamics simulations to investigate the effect of temperature on the stability and dynamics of the complex, and determine the changes in binding free energy at different temperatures.",
"How does the use of Monte Carlo simulations enhance the understanding of the mechanisms and kinetics of gas-phase reactions in atmospheric chemistry?\n", "Using principles of inorganic chemistry, predict the stability of the following two inorganic compounds: \n1) Sodium chloride (NaCl)\n2) Ammonium chloride (NH4Cl) \n\nBriefly explain the factors contributing to the stability of each compound based on their structures and bonding."
\end{Verbatim}
\end{tcolorbox}

\subsection{Example Prompts in Cell \texttt{ECC\_1}$\cap$\texttt{Emb\_10}}
\begin{tcolorbox}[left=1mm, right=1mm, top=1mm, bottom=1mm, boxrule=0.9pt, breakable]
\begin{Verbatim}[breaklines=true]
"What are the effects of doping on the electrical and optical properties of semiconductors, and how does the dopant concentration affect these properties?",
"How does the calculated bandgap of a material using density functional theory (DFT) compare to its experimental value? Can DFT calculations accurately predict the electronic properties of a material? Provide a comprehensive analysis of previous studies on this topic, including the challenges and limitations associated with DFT calculations of electronic properties of materials.",
"How can the electronic and magnetic properties of metal-organic frameworks (MOFs) be predicted using quantum chemistry methods such as density functional theory (DFT)?",
"What is the effect of adding different types of impurity scattering on the electronic and transport properties of graphene nanoribbon using Density Functional Theory (DFT) calculations?",
"How does the electronic transport property (conductivity) of a graphene nanoribbon change as the width of the ribbon is varied, as predicted by density functional theory (DFT) calculations?",
"How does the accuracy of Density Functional Theory (DFT) calculations affect the prediction of surface properties in catalysis reactions?",
"How does the choice of exchange-correlation functional in density functional theory (DFT) calculations impact the prediction of the electronic properties of materials such as band gap, optical absorption spectra, and charge carrier mobility? Compare the results obtained using different exchange-correlation functionals and discuss their advantages and limitations in predicting the electronic properties of materials.",
"How does the size and shape of a nanoparticle affect its behavior and stability under varying environmental conditions? Using molecular dynamics simulations, investigate the changes in nanoparticle properties such as surface area, reactivity, and structural integrity as different geometries and sizes are explored. Is there an optimal size and shape for nanoparticles to exhibit maximum stability and desirable properties, and how does this affect their potential applications in fields such as medicine, catalysis, and electronics?",
"How do quantum dots' photochemical properties change with variations in their size, shape, and composition, and how can this knowledge be applied to enhance their applications in areas such as biological imaging and solar cell technology?",
"How can the optical and electrical properties of liquid crystals be tuned to improve the efficiency and performance of liquid crystal displays?",
"How does the size of a quantum dot affect its electronic and optical properties? Provide a detailed explanation supported by calculations using quantum chemistry methods.",
"What is the relationship between the electronic and optical properties of perovskite materials, and how can this relationship be accurately calculated using quantum chemistry methods?",
"What is the effect of doping metal oxides with transition metals on the electronic properties of the material using density functional theory calculations?",
"How does temperature affect the phase transition of water from solid to liquid using Monte Carlo simulations?"
\end{Verbatim}
\end{tcolorbox}

\subsection{Example Prompts in Cell \texttt{ECC\_10}$\cap$\texttt{Emb\_10}}
\begin{tcolorbox}[left=1mm, right=1mm, top=1mm, bottom=1mm, boxrule=0.9pt, breakable]
\begin{Verbatim}[breaklines=true]
"How can the synthesis of polymer-based superabsorbent materials be optimized to achieve the highest absorption capacity and stability under different environmental conditions?",
"How can we improve the efficiency of extrusion processing for a specific polymer blend used in the manufacturing of plastic bottles? What changes in the processing parameters or equipment can be made to achieve a higher production rate and better quality of the final product while minimizing waste?",
"How can Dynamic Mechanical Analysis (DMA) be used to determine the glass transition temperature of a polymer? Provide a step-by-step explanation of the DMA technique and its application to polymer characterization.",
"How can we determine the molecular weight distribution of a polymer using gel permeation chromatography (GPC)?",
"How can we use gel permeation chromatography (GPC) to determine the molecular weight distribution of a polymer sample?",
"How does the degree of crystallinity of a polymer affect its mechanical properties?",
"How can we optimize the printing conditions for a specific polymer-based material to achieve the desired mechanical and functional properties in 3D printed parts?",
"How can the permeability of polymer-based membranes used in water treatment be optimized while maintaining high selectivity to remove contaminants?",
"How does the size and shape of a gold nanoparticle influence its surface plasmon resonance frequency? Use molecular dynamics simulations to investigate the behavior of different sized and shaped gold nanoparticles in a solvent environment.",
"How does changing the size of gold nanoparticles affect their thermodynamic stability when interacting with a solvent molecule? Use molecular dynamics simulations to investigate the differences in stability between nanoparticles with diameters of 5 nm, 10 nm, and 15 nm.",
"How do different simulation parameters (e.g. temperature, pressure, chain length) affect the behavior of polymer systems when using the Monte Carlo simulation method?"
\end{Verbatim}
\end{tcolorbox}


\end{document}